\documentclass{article}

\PassOptionsToPackage{numbers, compress}{natbib}
\usepackage[final]{nips_2017}

\usepackage[utf8]{inputenc}
\usepackage[T1]{fontenc}    
\usepackage[pdftex]{graphicx}
\usepackage{amsmath}
\usepackage{upgreek}
\usepackage{amsfonts,amssymb}
\usepackage{bm}
\usepackage{bbm}

\usepackage{amsthm} 
\usepackage{thmtools}
\usepackage{mathtools}
\usepackage{epstopdf}
\usepackage{xspace}
\usepackage[font=footnotesize]{caption}
\usepackage[font+=small]{subcaption} 
\usepackage{units}
\usepackage{booktabs} 
\usepackage{tabulary}
\usepackage[usenames,dvipsnames]{color} 
\usepackage{enumitem}
\usepackage{diagbox}
\usepackage[ruled,vlined,linesnumbered]{algorithm2e}  
\usepackage{parskip}
\SetKwComment{Comment}{$\triangleright$\ }{}

\usepackage{ifthen,version}
\usepackage{soul}
\usepackage{fixltx2e}
\usepackage{csquotes}
\usepackage{microtype}      
\usepackage{lipsum}

\usepackage{tikz}
\usepackage[page,title]{appendix}
\usepackage[pdftex,colorlinks,bookmarks=true]{hyperref} 
\newcommand{\fullFigGap}[0]{\vspace{-1.5\baselineskip}} 
\newcommand{\red}[0]{\color{red}}

\newtheorem{theorem}{Theorem}
\newtheorem*{theorem*}{Theorem}
\newtheorem{lemma}{Lemma}
\newtheorem*{lemma*}{Lemma}
\newtheorem{problem}{Problem}
\newtheorem{definition}{Definition}

\newtheorem{observation}{O}
\newcommand{\algorithmStyle}[0]{\footnotesize}

\DeclareMathOperator*{\argmin}{arg\,min}
\DeclareMathOperator*{\argmax}{arg\,max}

\newcommand{\argmaxprob}[1]{\underset{#1}{\argmax}}
\newcommand{\argminprob}[1]{\underset{#1}{\argmin}}

\newcommand{\abs}[1]{\left|#1 \right|}

\newcommand{\expect}[2]{\mathbb{E}_{#1}\left[#2\right]}

\newcommand{\real}[0]{\mathbb{R}}

\newcommand{\bbm}{\begin{bmatrix}}
\newcommand{\ebm}{\end{bmatrix}}

\newcommand{\pair}[2]{\left( #1, #2\right)}
\newcommand{\set}[1]{\left\lbrace #1\right\rbrace}
\newcommand{\seq}[2]{\left(#1_{1}, #1_{2}, \ldots, #1_{#2}\right)}

\newcommand{\setst}[2]{\left\lbrace #1\;\middle|\;#2\right\rbrace}

\newcommand{\Ind}[0]{\mathbb{I}}
\newcommand{\overbar}[1]{\mkern 1.5mu\overline{\mkern-1.5mu#1\mkern-1.5mu}\mkern 1.5mu}
\newcommand{\Not}[1]{\overbar{#1}}

\newcommand\encircle[1]{%
  \tikz[baseline=(X.base)] 
    \node (X) [draw, shape=circle, inner sep=0] {\strut #1};}

\newcommand{\explicitGraph}[0]{G}
\newcommand{\vertexSet}[0]{V}
\newcommand{\edgeSet}[0]{E}
\newcommand{\edge}[0]{e}
\newcommand{\start}[0]{v_s}
\newcommand{\goal}[0]{v_g}
\newcommand{\Path}[0]{\xi}
\newcommand{\PathSet}[0]{\Xi}
\newcommand{\evalFn}[0]{\mathtt{Eval}}
\newcommand{\world}[0]{\mathbf{o}}

\newcommand{\selectFn}[0]{\mathtt{Select}}

\newcommand{\test}[0]{t}
\newcommand{\testSet}[0]{\mathcal{T}}
\newcommand{\numTest}[0]{n}

\newcommand{\outcomeSpace}[0]{\{0, 1\}}
\newcommand{\outcomeVar}[0]{X}
\newcommand{\outcomeVarTest}[1]{\outcomeVar_{#1}}
\newcommand{\outcome}[0]{x}
\newcommand{\outcomeTest}[1]{\outcome_{#1}}
\newcommand{\outcomeTestSet}[1]{\mathbf{\outcome}_{#1}}

\newcommand{\bias}[0]{\theta}
\newcommand{\biasVec}[0]{\bm{\bias}}
\newcommand{\biasTest}[1]{\bias_{#1}}

\newcommand{\region}[0]{\mathcal{R}}
\newcommand{\numRegion}[0]{m}

\newcommand{\selTestSet}[0]{\mathcal{A}}
\newcommand{\obsOutcome}[0]{\outcomeTestSet{\selTestSet}}
\newcommand{\obsOutcomeFunc}[2]{\obsOutcome \left( #1, #2\right)}
\newcommand{\obsOutcomeAdd}[1]{\outcomeTestSet{\selTestSet \cup \set{#1}}}

\newcommand{\policy}[0]{\pi}
\newcommand{\groundtruth}[0]{\outcomeTestSet{\testSet}}

\newcommand{\cost}[0]{c}
\newcommand{\policyOpt}[0]{\pi^*}

\newcommand{\hyp}[0]{h}
\newcommand{\hypSpace}[0]{\mathcal{H}}
\newcommand{\hypSpaceR}[0]{\mathcal{H}^\region}
\newcommand{\regionH}[0]{\region^\hypSpace}

\newcommand{\subregion}[0]{\mathcal{S}}
\newcommand{\numSubregion}[0]{l}

\newcommand{\edgeEC}[0]{\varepsilon}
\newcommand{\edgeSetEC}[0]{\mathcal{E}}
\newcommand{\edgeFnEC}[2]{\{#1, #2\}}

\newcommand{\weight}[0]{w}

\newcommand{\ec}[0]{f_\mathrm{EC}}
\newcommand{\fec}[1]{f_\mathrm{EC}(#1)}

\newcommand{\gain}[3]{\Delta_{#1} \left( #2 \;|\; #3 \right)}

\newcommand{\wec}[0]{w_\mathrm{EC}}
\newcommand{\wecgolovin}[0]{w_{\text{\citep{golovin2010near}}}}

\newcommand{\ectext}[0]{\mathrm{EC}}
\newcommand{\drd}[0]{f_\mathrm{DRD}}
\newcommand{\fdrd}[1]{f_\mathrm{DRD}(#1)}
\newcommand{\feci}[2]{f_\mathrm{EC}^{#1}(#2)}

\newcommand{\candTestSet}[0]{\testSet_\mathrm{cand}}
\newcommand{\optTest}[0]{t^*}
\newcommand{\maxProbTestSet}[0]{\testSet_\mathrm{maxP}}

\newcommand{\degree}[0]{k}

\newcommand{\nv}[0]{n_\mathcal{V}}
\newcommand{\nvec}[0]{\mathbf{n}}

\newcommand{\pmin}[0]{p_\mathrm{min}}
\newcommand{\pminH}[0]{p_\mathrm{min}^h}

\newcommand{\dataTrain}[0]{N_{\mathrm{train}}}
\newcommand{\dataTest}[0]{N_{\mathrm{test}}}
\newcommand{\regTest}[0]{\mathbf{A}}

\newcommand{\bigo}[1]{\mathcal{O}\left(#1\right)}

\newcommand{\direct}[0]{\textsc{DiRECt}\xspace}
\newcommand{\ecsq}[0]{EC\textsuperscript{2}\xspace}
\newcommand{\algName}[0]{\textsc{BiSECt}\xspace}
\newcommand{\algFullName}[0]{Bernoulli Subregion Edge Cutting\xspace}

\newcommand{\algRandom}[0]{\textsc{Random}\xspace}
\newcommand{\algMaxTally}[0]{\textsc{MaxTally}\xspace}
\newcommand{\algSetCover}[0]{\textsc{SetCover}\xspace}
\newcommand{\algMVOI}[0]{\textsc{MVoI}\xspace}

\newcommand{\probBernDRD}[0]{Bern-DRD\xspace}
\newcommand{\probBernDRDFull}[0]{decision region determination with independent Bernoulli tests\xspace}

\newcommand{\algMaxProbReg}[0]{\textsc{MaxProbReg}\xspace}

\graphicspath{{figs/}}

\title{Near-Optimal Edge Evaluation in Explicit Generalized Binomial Graphs}

\author{
  Sanjiban Choudhury \\
  The Robotics Institute\\
  Carnegie Mellon University\\
  \texttt{sanjiban@cmu.edu} \\
  \And
  Shervin Javdani \\
  The Robotics Institute\\
  Carnegie Mellon University\\
  \texttt{sjavdani@cmu.edu} \\
  \AND
  Siddhartha Srinivasa \\
  The Robotics Institute\\
  Carnegie Mellon University\\
  \texttt{siddh@cs.cmu.edu} \\
  \And
  Sebastian Scherer \\
  The Robotics Institute\\
  Carnegie Mellon University\\
  \texttt{basti@cs.cmu.edu} \\
}

\begin{document}
\maketitle
\makeatletter
\renewcommand{\section}{%
\@startsection{section}{1}{\z@}%
            {-2.0ex \@plus -0.0ex \@minus -1.0ex}%
            { 1.5ex \@plus  0.0ex \@minus  0.5ex}%
            {\large\bf\raggedright}%
}
\providecommand{\subsection}{}
\renewcommand{\subsection}{%
  \@startsection{subsection}{2}{\z@}%
                {-1.2ex \@plus -0.0ex \@minus -0.8ex}%
                { 0.8ex \@plus  0.0ex }%
                {\normalsize\bf\raggedright}%
}
\providecommand{\subsubsection}{}
\renewcommand{\subsubsection}{%
  \@startsection{subsubsection}{3}{\z@}%
                {-0.8ex \@plus -0.0ex \@minus -0.5ex}%
                { 0.5ex \@plus  0.0ex}%
                {\normalsize\bf\raggedright}%
}
\makeatother
\setlength \abovedisplayskip{1pt plus0pt minus1pt}
\setlength \belowdisplayskip{\abovedisplayskip}
\setlength\abovecaptionskip{0.pt plus0pt minus0pt}                                      

\begin{abstract}

Robotic motion-planning problems, such as a UAV flying fast in a partially-known environment or a robot arm moving around cluttered objects, require finding collision-free paths quickly. 
Typically, this is solved by constructing a graph, where vertices represent robot configurations and edges represent potentially valid movements of the robot between these configurations.
The main computational bottlenecks are expensive edge evaluations to check for collisions.
State of the art planning methods do not reason about the \emph{optimal sequence of edges} to evaluate in order to find a collision
free path quickly.
In this paper, we do so by drawing a novel equivalence between motion planning and the Bayesian active learning paradigm of \emph{decision region determination (DRD)}. 
Unfortunately, a straight application of existing methods requires computation exponential in the number of edges in a graph. 
We present \algName, an efficient and near-optimal algorithm to solve the DRD problem when edges are independent Bernoulli random variables. 
By leveraging this property, we are able to significantly reduce computational complexity from exponential to linear in the number of edges.
We show that \algName outperforms several state of the art algorithms on a spectrum of planning problems for mobile robots, manipulators, and real flight data collected from a full scale helicopter.
\end{abstract}

\section{Introduction}
\label{sec:intro}
This paper addresses a class of robotic motion planning problems where \emph{path evaluation is expensive}. 
For example, in robot arm planning~\citep{dellin2016guided}, evaluation requires expensive geometric intersection computations. In on-board path planning for UAVs with limited computational resources~\citep{cover2013sparse}, the system must react quickly to obstacles (Fig.~\ref{fig:marquee}).

State of the art planning algorithms~\citep{dellin2016unifying} first compute a set of unevaluated paths quickly, and then evaluate them sequentially to find a valid path. 
Oftentimes, candidate paths share common edges. Hence, evaluation of a small number of edges can provide information about the validity of many candidate paths simultaneously. Methods that check paths sequentially, however, do not reason about these common edges.

This leads us naturally to the \emph{feasible path identification} problem - given a library of candidate paths, identify a valid path while minimizing the cost of edge evaluations. We assume access to a prior distribution over edge validity, which encodes how obstacles are distributed in the environment (Fig.~\ref{fig:marquee}(a)). As we evaluate edges and observe outcomes, the uncertainty of a candidate path collapses. 

Our first key insight is that this problem is equivalent to \emph{decision region determination (DRD)}~\citep{javdani2014near, chen2015submodular}) - given a set of tests (edges), hypotheses (validity of edges), and regions (paths), the objective is to drive uncertainty into a single decision region. This linking enables us to leverage existing methods in Bayesian active learning for robotic motion planning.

\citet{chen2015submodular} provide a method to solve this problem by maximizing an objective function that satisfies \emph{adaptive submodularity}~\citep{golovin2011adaptive} - a natural diminishing returns property that endows greedy policies with near-optimality guarantees. Unfortunately, naively applying this algorithm requires $\bigo{2^E}$ computation to select an edge to evaluate, where $E$ is the number of edges in all paths.

We define the \probBernDRD problem, which leverages additional structure in robotic motion planning by assuming edges are independent Bernoulli random variables~\footnote{Generally, edges in this graph are correlated, as edges in collision are likely to have neighbours in collision. Unfortunately, even measuring this correlation is challenging, especially in the high-dimensional non-linear configuration space of robot arms. Assuming independent edges is a common simplification~\citep{Lav06, narayanan2017heuristic, choudhury2016pareto, burns2005sampling,dellin2016unifying}}, and regions correspond to sets of edges evaluating to true. We propose \emph{\algFullName (\algName)}, which provides a greedy policy to select candidate edges in $\bigo{E}$. We prove our surrogate objective also satisfies adaptive submodularity~\citep{golovin2011adaptive}, and provides the same bounds as \citet{chen2015submodular} while being more efficient to compute. 

\begin{figure}[!t]
    \centering
    \includegraphics[page=1,width=0.8\textwidth]{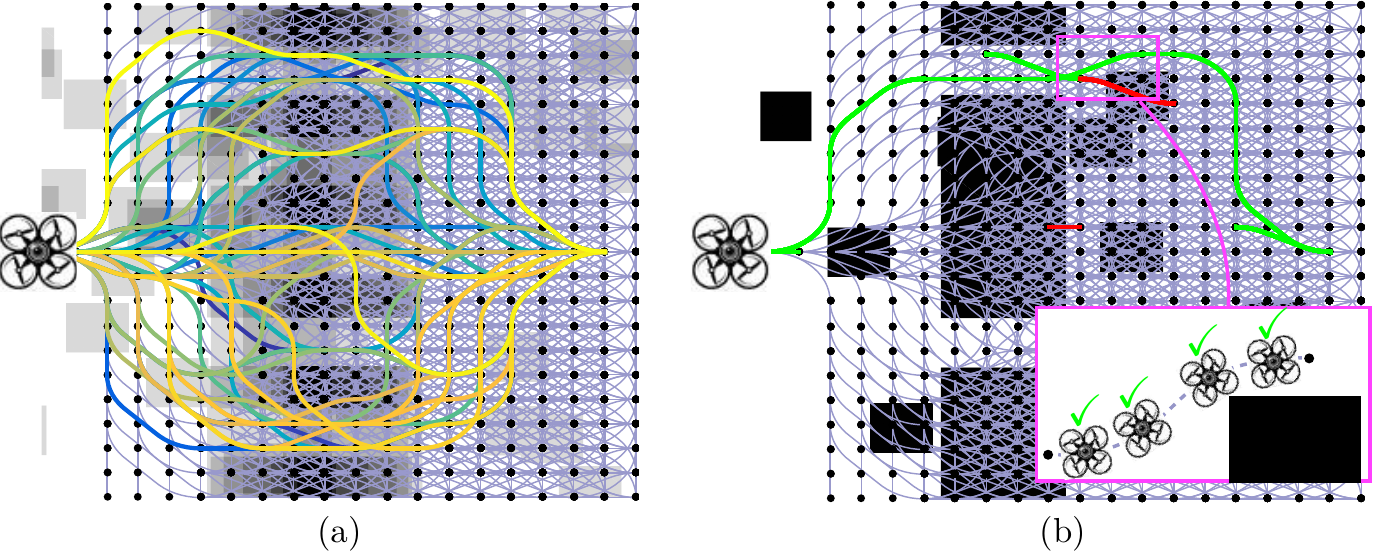}
    \caption{%
    \label{fig:marquee}
    The feasible path identification problem (a) The explicit graph contains dynamically feasible maneuvers~\citep{pivtoraiko2009differentially} for a UAV flying fast, with a set candidate paths. The map shows the distribution of edge validity for the graph. (b) Given a distribution over edges, our algorithm checks an edge, marks it as invalid (red) or valid (green), and updates its belief. We continue until a feasible path is identified as free. We aim to minimize the number of expensive edge evaluations. 
    \fullFigGap}
\end{figure}%

We make the following contributions:
\begin{enumerate}
  \item We show a novel equivalence between feasible path identification and the DRD problem, linking motion planning to Bayesian active learning.
  \item We develop \algName, a near-optimal algorithm for the special case of Bernoulli tests, which selects tests in $\bigo{E}$ instead of $\bigo{2^E}$.
  \item We demonstrate the efficacy of our algorithm on a spectrum of planning problems for mobile robots, manipulators, and real flight data collected from a full scale helicopter. 
\end{enumerate}

\section{Problem Formulation}
\label{sec:problem_formulation}

\subsection{Planning as Feasible Path Identification on Explicit Graphs}
Let $\explicitGraph = \pair{\vertexSet}{\edgeSet}$ be an explicit graph that consists of a set of vertices $\vertexSet$ and edges $\edgeSet$. Given a pair of start and goal vertices, $\pair{\start}{\goal} \in \vertexSet$, a search algorithm computes a path $\Path \subseteq \edgeSet$ - a connected sequence of valid edges. To ascertain the validity of an edge, it invokes an evaluation function $\evalFn: \edgeSet \rightarrow \{0, 1\}$. We address applications where edge evaluation is expensive, i.e., the computational cost $\cost(\edge)$ of computing $\evalFn(\edge)$ is significantly higher than regular search operations\footnote{It is assumed that $\cost(\edge)$ is modular and non-zero. It can scale with edge length.}.

We define a world as an outcome vector $\world \in \{0, 1 \}^{\abs{\edgeSet}}$ which assigns to each edge a boolean validity when evaluated, i.e. $\evalFn(\edge) = \world(\edge)$.
We assume that the outcome vector $P(\world)$ is sampled from an independent Bernoulli distribution, giving rise to a \emph{Generalized Binomial Graph (GBG) }~\citep{frieze2015introduction}.

We make a second simplification to the problem - from that of search to that of identification. Instead of searching $\explicitGraph$ online for a path, we frame the problem as identifying a valid path from a library of `good' candidate paths $\PathSet = \seq{\Path}{\numRegion}$. The candidate set of paths $\PathSet$ is constructed offline, while being cognizant of $P(\world)$, and can be verified to ensure that all paths have acceptable solution quality when valid. \footnote{Refer to supplementary on various methods to construct a library of good candidate paths} Hence we care about completeness with respect to $\PathSet$ instead of $\explicitGraph$. 

We wish to design an adaptive edge selector $\selectFn(\world)$ which is a decision tree that operates on a world $\world$, selects an edge for evaluation and branches on its outcome. The total cost of edge evaluation is $\cost(\selectFn(\world))$. Our objective is to minimize the cost required to find a valid path:
\begin{equation}
\label{eq:e4gbg}
\min \; \expect{ \world \in P(\world) }{ \cost(\selectFn(\world)) }  \; \mathrm{s.t} \; \forall \world , \exists \Path \; : \; \prod\limits_{\edge \in \Path} \world(\edge) = 1 \; , \; \Path \subseteq \selectFn(\world)
\end{equation}

\subsection{Decision Region Determination with Independent Bernoulli Tests}

We now define an equivalent problem - \emph{\probBernDRDFull (\probBernDRD)}. Define a set of tests $\testSet = \{1,\dots,\numTest\}$, where the outcome of each test is a Bernoulli random variable $\outcomeVarTest{\test} \in \outcomeSpace$, $P(\outcomeVarTest{\test} = \outcomeTest{\test}) = \biasTest{\test}^{\outcomeTest{\test}} (1 - \biasTest{\test})^{1 - \outcomeTest{\test}}$. We define a set of hypotheses $\hyp \in \hypSpace$, where each is an outcome vector $\hyp \in \outcomeSpace^\testSet$ mapping all tests $\test \in \testSet$ to outcomes $\hyp(\test)$. We define a set of regions $\set{\region_i}_{i=1}^{\numRegion}$, each of which is a subset of tests $\region \subseteq \testSet$. A region is determined to be valid if all tests in that region evaluate to true, which has probability $P(\region) = \prod\limits_{\test \in \region} P(\outcomeVarTest{\test} = 1)$. \vspace{-0.5em}

If a set of tests $\selTestSet \subseteq \testSet$ are performed, let the observed outcome vector be denoted by $\obsOutcome \in \outcomeSpace^{\abs{\selTestSet}}$. Let the \emph{version space} $\hypSpace(\obsOutcome)$ be the set of hypotheses consistent with observation vector $\obsOutcome$, i.e. $\hypSpace(\obsOutcome) = \setst{\hyp \in \hypSpace}{ \forall \test \in \selTestSet, \hyp(\test) = \obsOutcome(\test)}$.

We define a policy $\policy$ as a mapping from observation vector $\obsOutcome$ to tests. A policy terminates when it shows that at least one region is valid, or all regions are invalid. Let $\groundtruth \in \outcomeSpace^\testSet$ be the ground truth - the outcome vector for all tests. Denote the observation vector of a policy $\policy$ given ground truth $\groundtruth$ as $\obsOutcomeFunc{\policy}{\groundtruth}$. The expected cost of a policy $\policy$ is $\cost(\pi) = \expect{\groundtruth}{\cost(\obsOutcomeFunc{\policy}{\groundtruth}}$ where $c(\obsOutcome)$ is the cost of all tests $\test \in \selTestSet$. The objective is to compute a policy $\policyOpt$ with minimum cost that ensures at least one region is valid, i.e.
\begin{equation}
\label{eq:drd}
\policyOpt \in \argminprob{\policy} \;\cost(\policy) \; \mathrm{s.t} \; \forall \groundtruth , \exists \region_d \; : \; P(\region_d \;|\; \obsOutcomeFunc{\policy}{\groundtruth}) = 1
\end{equation}
Note that we can cast problem (\ref{eq:e4gbg}) to (\ref{eq:drd}) by setting $\edgeSet=\testSet$ and $\PathSet=\set{\region_i}_{i=1}^{\numRegion}$. That is, driving uncertainty into a region is equivalent to identification of a valid path (Fig.~\ref{fig:bern_drd_problem}). This casting enables us to leverage efficient algorithms with near-optimality guarantees for motion planning. 

\section{Related Work} \vspace{-0.7em}
\label{sec:related_work}
The computational bottleneck in motion planning varies with problem domain and that has led to a plethora of planning techniques (\citep{Lav06}). When vertex expansions are a bottleneck, A*~\citep{hart1968formal} is optimally efficient while techniques such as partial expansions~\citep{yoshizumi2000partial} address graph searches with large branching factors. The problem class we examine, that of expensive edge evaluation, has inspired a variety of `lazy' approaches. The Lazy Probabilistic Roadmap (PRM) algorithm~\citep{bohlin2000path} only evaluates edges on the shortest path while Fuzzy PRM~\citep{nielsen2000two} evaluates paths that minimize probability of collision. The Lazy Weighted A* (LWA*) algorithm~\citep{cohen2015planning} delays edge evaluation in A* search and is reflected in similar techniques for randomized search~\citep{gammell2015batch,Choudhury_2016_8070}. An approach most similar in style to ours is the LazyShortestPath (LazySP) framework~\citep{dellin2016unifying} which examines the problem of which edges to evaluate on the shortest path. Instead of the finding the shortest path, our framework aims to efficiently identify a feasible path in a library of `good' paths. Our framework is also similar to the Anytime Edge Evaluation (AEE*) framework~\citep{narayanan2017heuristic} which deals with edge evaluation on a GBG. However, our framework terminates once a \emph{single feasible path} is found while AEE* continues to evaluation in order to minimize expected cumulative sub-optimality bound. Similar to \citet{choudhury2016pareto} and \citet{burns2005sampling}, we leverage priors on the distribution of obstacles to make informed planning decisions.

\begin{figure}[!t]
    \centering
    \includegraphics[page=1,width=\textwidth]{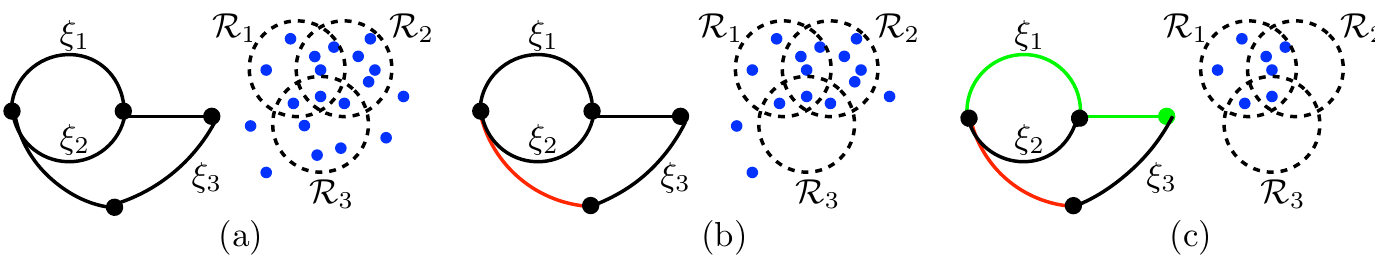}
    \caption{%
    \label{fig:bern_drd_problem}
    Equivalence between the feasible path identification problem and \probBernDRD. A path $\Path_i$ is equivalent to a region $\region_i$ over valid hypotheses (blue dots). Tests eliminate hypotheses and the algorithm terminates when uncertainty is pushed into a region ($\region_1$) and the corresponding path ($\Path_1$) is determined to be valid. \fullFigGap}
\end{figure}%


We draw a novel connection between motion planning and optimal test selection which has a wide-spread application in medical diagnosis~\citep{kononenko2001machine} and experiment design~\citep{chaloner1995bayesian}. Optimizing the ideal metric, decision theoretic value of information~\citep{howard1966information}, is known to be NP\textsuperscript{PP} complete~\citep{krause2009optimal}. For hypothesis identification (known as the Optimal Decision Tree (ODT) problem), Generalized Binary Search (GBS)~\citep{dasgupta2004analysis} provides a near-optimal policy. For disjoint region identification (known as the Equivalence Class Determination (ECD) problem), \ecsq~\citep{golovin2010near} provides a near-optimal policy. When regions overlap (known as the Decision Region Determination (DRD) problem), HEC~\citep{javdani2014near} provides a near-optimal policy. The \direct algorithm~\citep{chen2015submodular}, a computationally more efficient alternative to HEC, forms the basis of our approach.

\section{The Bernoulli Subregion Edge Cutting Algorithm} \label{sec:algorithm}
We follow the framework of \emph{Decision Region Edge Cutting (\direct)}~\citep{chen2015submodular} by creating separate sub-problems for each region, and combining them. For each sub-problem, we provide a modification to \ecsq which is simpler to compute when the distribution over hypotheses is non-uniform, while providing the same guarantees. Unfortunately, naively applying this method requires $\bigo{2^\testSet}$ computation per sub-problem. For the special case of independent Bernoulli tests, we present a more efficient \emph{Bernoulli Subregion Edge Cutting (\algName)} algorithm, which computes each subproblem in $\bigo{\testSet}$ time. 

\subsection{Preliminaries: Hypothesis as outcome vectors}
In order to apply the DRD framework of \citet{chen2015submodular}, we need to view regions as a sets of hypotheses.
 A hypothesis  $\hyp$ is a mapping from a test $\test \in \testSet$ to an outcome $\hyp(\test)$ and is defined as an outcome vector $\hyp \in \outcomeSpace^\testSet$. We use the symbol $\hypSpace$ to denote the set of all hypothesis ($\hypSpace = \outcomeSpace^\testSet$). Using the independent Bernoulli distribution, the probability of a hypothesis is $P(\hyp) = \prod\limits_{\test \in \testSet} P(\outcomeVarTest{\test} = \hyp(\test)) = \prod\limits_{\test \in \testSet} \biasTest{\test}^{\hyp(\test)} (1 - \biasTest{\test})^{1 - \hyp(\test)}$.

Given a observation vector $\obsOutcome$, let the \emph{version space} $\hypSpace(\obsOutcome)$ be the set of hypothesis consistent with $\obsOutcome$, i.e. $\hypSpace(\obsOutcome) = \setst{\hyp \in \hypSpace}{ \forall \test \in \selTestSet, \hyp(\test) = \obsOutcome(\test)}$. The probability mass of all the version space can evaluated as 
$P(\hypSpace(\obsOutcome)) = 
\sum\limits_{\hyp \in \hypSpace(\obsOutcome)} P(\hyp) = 
\prod\limits_{i \in \selTestSet} \biasTest{i}^{\obsOutcome(i)} (1 - \biasTest{i})^{1 - \obsOutcome(i)}  $ 

Although we initially defined a region as a clause on constituent test outcomes being true, we can now view them as a version space consistent with the constituent tests. Hence given a region $\region$, we define the version space $\regionH \in \hypSpace$ as a set of consistent hypothesis
$\regionH = \setst{\hyp \in \hypSpace}{\forall \test \in \region, \hyp(\test) = 1}$
Hence the probability of a region being valid is the probability mass of all consistent hypothesis
 $P(\regionH)  =  \sum\limits_{\hyp \in \regionH} P(\hyp) =  \prod\limits_{i \in \region} P(\outcomeVarTest{i} = 1) =  \prod\limits_{i \in \region} \biasTest{i} $

We will now define a set of useful expressions that will be used by \algName. Given a observation vector $\obsOutcome$, the \emph{relevant version space} is denoted as $\hypSpaceR(\obsOutcome) = \setst{\hyp \in \hypSpace}{ \forall \test \in \selTestSet \cap \region, \hyp(\test) = \obsOutcome(\test)}$. Hence the set of all hypothesis in $\regionH$ consistent with relevant outcomes in $\obsOutcome$ is given by $\regionH \cap \hypSpaceR(\obsOutcome)$. The probability $P(\regionH \cap \hypSpaceR(\obsOutcome))$ is as follows

\begin{equation}
\begin{aligned}
\label{eq:p_region_prune}
P(\regionH \cap \hypSpaceR(\obsOutcome))  &=  \sum\limits_{\hyp \in \regionH \cap \hypSpaceR(\obsOutcome)} P(\hyp) \\
                    &=  \sum\limits_{\hyp \in \regionH \cap \hypSpaceR(\obsOutcome)} \prod\limits_{i \in \testSet} P(\outcomeVarTest{i} = \hyp(i)) \\
                    &=  \prod\limits_{i \in (\region \cap \selTestSet)} \Ind(\outcomeVarTest{i} = 1)
                      \prod\limits_{j \in (\region \setminus \selTestSet)} P(\outcomeVarTest{j} = 1) 
                      \prod\limits_{k \in \region \cap \selTestSet} P(\outcomeVarTest{k} = \obsOutcome(k)) \\
                    &=  \prod\limits_{i \in (\region \cap \selTestSet)} \Ind(\outcomeVarTest{i} = 1)
                      \prod\limits_{j \in (\region \setminus \selTestSet)} \biasTest{j}
                      \prod\limits_{k \in \region \cap \selTestSet} \biasTest{k}^{\obsOutcome(k)} (1 - \biasTest{k})^{1 - \obsOutcome(k)} \\
\end{aligned}
\end{equation}

We will now derive similar expressions for the probability of a region \emph{not} being valid. The probability mass of hypothesis where a region $\region$ is not valid is $P(\Not{\regionH}) = \sum\limits_{\hyp \in \Not{\regionH}} P(\hyp) = 1 - \prod\limits_{i \in \region} \biasTest{i}$

Similarly, the set of all hypothesis in $\Not{\regionH}$ consistent with relevant outcomes in $\obsOutcome$ is given by $\Not{\regionH} \cap \hypSpaceR(\obsOutcome)$. The probability $P(\Not{\regionH} \cap \hypSpaceR(\obsOutcome))$ is as follows

\begin{equation}
\begin{aligned}
\label{eq:p_notregion_prune}
P(\Not{\regionH} \cap \hypSpaceR(\obsOutcome))  &=  \sum\limits_{\hyp \in \Not{\regionH} \cap \hypSpaceR(\obsOutcome)} P(\hyp) \\
                    &=  \sum\limits_{\hyp \in \Not{\regionH} \cap \hypSpaceR(\obsOutcome)} \prod\limits_{i \in \testSet} P(\outcomeVarTest{i} = \hyp(i)) \\
                    &=  \left(1 - 
                      \prod\limits_{i \in (\region \cap \selTestSet)} \Ind(\outcomeVarTest{i} = 1)
                      \prod\limits_{j \in (\region \setminus \selTestSet)} P(\outcomeVarTest{j} = 1) 
                      \right)
                      \prod\limits_{k \in \region \cap \selTestSet} P(\outcomeVarTest{k} = \obsOutcome(k)) \\
                    &=  \left(1 - 
                      \prod\limits_{i \in (\region \cap \selTestSet)} \Ind(\outcomeVarTest{i} = 1)
                      \prod\limits_{j \in (\region \setminus \selTestSet)} \biasTest{j}
                      \right)
                      \prod\limits_{k \in \region \cap \selTestSet} \biasTest{k}^{\obsOutcome(k)} (1 - \biasTest{k})^{1 - \obsOutcome(k)} \\
\end{aligned}
\end{equation}

\subsection{A simple subproblem: One region versus all}
We will now define a simple subproblem whose solution will help in addressing the \probBernDRD problem. We define the `one region versus all' subproblem as follows - given a \emph{single region}, the objective is to either push the entire probability mass of the version space on a region or collapse it on a single relevant hypothesis. We will view this as a decision problem on the space of \emph{disjoint subregions}.

We refer to hypothesis region $\regionH$ as subregion $\subregion_1$ as shown in Fig.\ref{fig:ecd_problem}. Every other hypothesis $\hyp \in \Not{\regionH}$ is defined as its own subregion $\subregion_i$. Determining which subregion is valid falls under the framework of \emph{Equivalence Class Determination} (ECD), (a special case of the DRD problem) and can be solved efficiently by the \ecsq algorithm (\citet{golovin2010near}). 

\subsubsection{The \texorpdfstring{\ecsq}{EC2} algorithm}
The ECD problem is a special case of the DRD problem described in (\ref{eq:drd}) to a case where regions are disjoint. In order to avoid confusion with DRD regions, we will hence forth refer to them as sub-regions.  Let $\{ \subregion_1, \dots, \subregion_\numSubregion \}$ be a set of disjoint subregions, i.e, $\subregion_i \cap \subregion_j = 0$ for $i \neq j$. \citet{golovin2010near} provide an efficient yet near-optimal criterion for solving ECD in their \ecsq algorithm which we discuss in brief. 

The \ecsq algorithm defines a graph $\mathcal{G}=(\mathcal{V}, \edgeSetEC)$ where the nodes are hypotheses and edges are between hypotheses in different decision regions $E = \cup_{i \neq j} \setst{ \edgeFnEC{\hyp}{\hyp'} }{\hyp \in \subregion_i, \hyp' \in \subregion_j}$. The weight of an edge is defined as $\weight(\edgeFnEC{\hyp}{\hyp'}) = P(\hyp) P(\hyp')$. The weight of a set of edges is defined as $\weight(\edgeSetEC') = \sum\limits_{\edgeEC \in \edgeSetEC'} \weight(\edgeEC)$. 
An edge is said to be `cut' by an observation if either hypothesis is inconsistent with the observation. Hence a test $\test$ with outcome $\outcomeTest{\test}$ is said to cut a set of edges $\edgeSetEC(\outcomeTest{\test}) = \setst{\edgeFnEC{\hyp}{\hyp'}}{\hyp(\test) \neq \outcomeTest{\test} \vee \hyp'(\test) \neq \outcomeTest{\test}}$. The aim is to cut all edges by performing test while minimizing cost. Before we describe the objective, we first specify how \ecsq efficiently computes weights by defininig a weight function over subregions.
\begin{equation}
\label{eq:weight_golovin}
\wecgolovin(\{\subregion_i\}) 
= \sum\limits_{i \neq j} P(\subregion_i) P(\subregion_j)
\end{equation}
When hypotheses have uniform weight, this can be computed efficiently for the `one region versus all' subproblem. Let $P(\Not{\subregion_1}) =  \sum\limits_{i>1} P(\subregion_i)$:
\begin{equation}
\label{eq:weight_golovin_simple}
    \wecgolovin(\{\subregion_i\}) = P(\subregion_1) P(\Not{\subregion_1}) + P(\Not{\subregion_1})\left(P(\Not{\subregion_1}) - \frac{1}{|\hypSpace|}\right)
\end{equation}

\ecsq defines an objective function $\fec{\obsOutcome}$ that measures the weight of edges cut. This is the difference between the original weight of subregions $\subregion_i$ and the weight of pruned subregions $\subregion_i \cap \hypSpace(\obsOutcome)$, i.e. $\fec{\obsOutcome} = \wecgolovin(\{\subregion_i\}) - \wecgolovin(\{\subregion_i\} \cap \hypSpace(\obsOutcome))$. 

\ecsq uses the fact that $\fec{\obsOutcome}$ is \emph{adaptive submodular} (\citet{golovin2011adaptive}) to define a greedy algorithm. Let the expected marginal gain of a test be 
$\gain{\ec}{\test}{\outcome} = \expect{\outcomeTest{\test}}{ \fec{ \obsOutcomeAdd{\test} } - \fec{\obsOutcome} \;|\; \obsOutcome}$. \ecsq greedily selects a test $\test^* \in \argmaxprob{\test} \frac{\gain{\ec}{\test}{\obsOutcome}}{\cost(\test)}$.

\vspace{-1.0\baselineskip}
\subsubsection{An alternative to \texorpdfstring{\ecsq}{EC2} on the `one region versus all' problem}

\begin{figure}[!htbp]
    \centering
    \includegraphics[page=1,width=\textwidth]{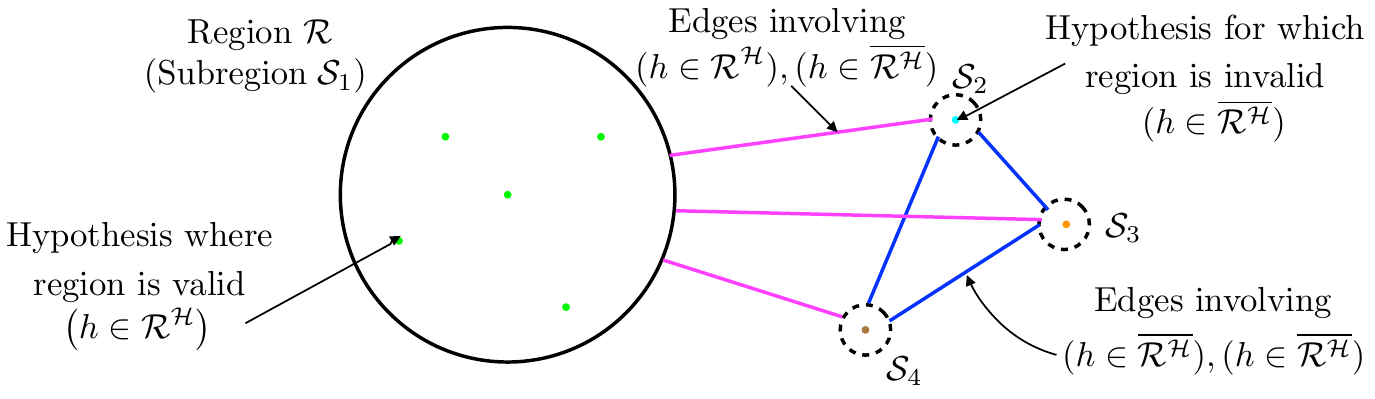}
    \caption{%
    \label{fig:ecd_problem}
    The `one region versus all' ECD problem. The region $\regionH$ is shown as a circle encompassing a set of consistent hypothesis $\hyp$ (green dots). Hypothesis for which the region is not valid lie outside the circle (dots in colors other than green). The objective is to compute an efficient policy to either force the probability mass in the region $\regionH$ or determine the \emph{unique} hypothesis $\hyp \in \Not{\regionH}$. \fullFigGap}
\end{figure}%

For non-uniform prior the quantity (\ref{eq:weight_golovin_simple}) is more difficult to compute. We modify this objective slightly, adding self-edges on subregions $\subregion_{i}, i>1$ as shown in Fig.~\ref{fig:ecd_problem}, enabling more efficient computation while still maintaining the same guarantees:
\begin{equation}
\label{eq:weight_sub}
\begin{aligned}
    \wec(\{\subregion_i\}) &= P(\subregion_1) (\sum\limits_{i\neq1} P(\subregion_i)) + (\sum\limits_{i\neq1} P(\subregion_i)) (\sum\limits_{j \geq i} P(\subregion_j)) \\
    &= P(\subregion_1) P(\Not{\subregion_1}) + P(\Not{\subregion_1})^2\\
    &= P(\regionH)P(\Not{\regionH}) + P(\Not{\regionH})P(\Not{\regionH}) \\
    &= P(\Not{\regionH}) (P(\regionH) + P(\Not{\regionH})) \\
    &= 1 - \prod\limits_{i \in \region} \biasTest{i}
\end{aligned}
\end{equation}

Similarly we can compute $\wec(\{\subregion_i\} \cap \hypSpaceR(\obsOutcome))$ using (\ref{eq:p_region_prune}) and (\ref{eq:p_notregion_prune})
\begin{equation}
\begin{aligned}
\label{eq:weight_sub_pruned}
&\wec(\{\subregion_i\} \cap \hypSpaceR(\obsOutcome)) \\
&=P(\subregion_1 \cap \hypSpaceR(\obsOutcome) ) P(\Not{\subregion_1} \cap \hypSpaceR(\obsOutcome) ) + P(\Not{\subregion_1} \cap \hypSpaceR(\obsOutcome))^2\\
   &= P(\region \cap \hypSpaceR(\obsOutcome) )P(\Not{\region} \cap \hypSpaceR(\obsOutcome)) + P(\Not{\region}\cap \hypSpaceR(\obsOutcome))P(\Not{\region}\cap \hypSpaceR(\obsOutcome)) \\
   &= P(\Not{\region}\cap \hypSpaceR(\obsOutcome)) (P(\region\cap \hypSpaceR(\obsOutcome)) + P(\Not{\region}\cap \hypSpaceR(\obsOutcome))) \\
   &= \left(1 - \prod\limits_{i \in (\region \cap \selTestSet)} \Ind(\outcomeVarTest{i} = 1)
      \prod\limits_{j \in (\region \setminus \selTestSet)} \biasTest{j} \right)
      \left( \prod\limits_{k \in \region \cap \selTestSet} \biasTest{k}^{\obsOutcome(k)} (1 - \biasTest{k})^{1 - \obsOutcome(k)} \right)^2
\end{aligned}
\end{equation}

Using (\ref{eq:weight_sub}) and (\ref{eq:weight_sub_pruned}) we can express the $\fec{\obsOutcome}$ as 

\begin{equation}
\begin{aligned}
\label{eq:fec_applied}
\fec{\obsOutcome} &= 1 - \frac{ \wec(\{\subregion_i\} \cap \hypSpaceR(\obsOutcome) ) }{ \wec(\{\subregion_i\}) }  \\
&= 1 - \frac{\left(1 - \prod\limits_{i \in (\region \cap \selTestSet)} \Ind(\outcomeVarTest{i} = 1)
      \prod\limits_{j \in (\region \setminus \selTestSet)} \biasTest{j} \right)
      \left( \prod\limits_{k \in \region \cap \selTestSet} \biasTest{k}^{\obsOutcome(k)} (1 - \biasTest{k})^{1 - \obsOutcome(k)} \right)^2
}{1 - \prod\limits_{i \in \region} \biasTest{i}}
\end{aligned}
\end{equation}

\begin{lemma} \label{lem:ec2}
The expression $\fec{\obsOutcome}$ is strongly adaptive monotone and adaptive submodular.
\end{lemma}
\begin{proof}
See Appendix \ref{sec:proof:lem_ec2}
\end{proof}

\subsection{Improvement in runtime from exponential to linear}
For non-uniform priors, computing (\ref{eq:weight_golovin}) is difficult. The naive approach is to compute all hypothesis and assign them to correct subregions and then compute the weights. This has a runtime of a runtime of $\bigo{2^\testSet}$.

However, our expression (\ref{eq:fec_applied}) can be computed in $\bigo{\testSet}$. This is because of the simplifications induced by the independent bernoulli assumption. 

Since we have to repeat this computation every iteration of the algorithm, we can reduce this to $\bigo{1}$ through memoization. If we memoize $\left(1 - \prod\limits_{i \in (\region \cap \selTestSet)} \Ind(\outcomeVarTest{i} = 1) \prod\limits_{j \in (\region \setminus \selTestSet)} \biasTest{j} \right)$, we can incrementally update it every time a test $t$ is evaluated. We also need to memoize $\left( \prod\limits_{k \in \region \cap \selTestSet} \biasTest{k}^{\obsOutcome(k)} (1 - \biasTest{k})^{1 - \obsOutcome(k)} \right)^2$ and update it incrementally. 

\subsection{Solving the original DRD problem using \algName }

We now return to the \probBernDRD (\ref{eq:drd}) where we have multiple regions $\{ \region_1, \dots, \region_\numRegion \}$ that can overlap and the goal is to push the probability into one such region. Similar to \direct (\citet{chen2015submodular}), we apply \algName to solve the problem.

\vspace{-0.5em}\subsubsection{The Noisy-OR Construction}\vspace{-0.5em}
The general strategy is to reduce the DRD problem with $\numRegion$ regions to $O(\numRegion)$ instances of the ECD problem such that \emph{solving any one of them} is sufficient for solving the DRD problem as shown in Fig.~\ref{fig:direct_problem}. 

\begin{figure}[!htbp]
    \centering
    \includegraphics[page=1,width=0.7\textwidth]{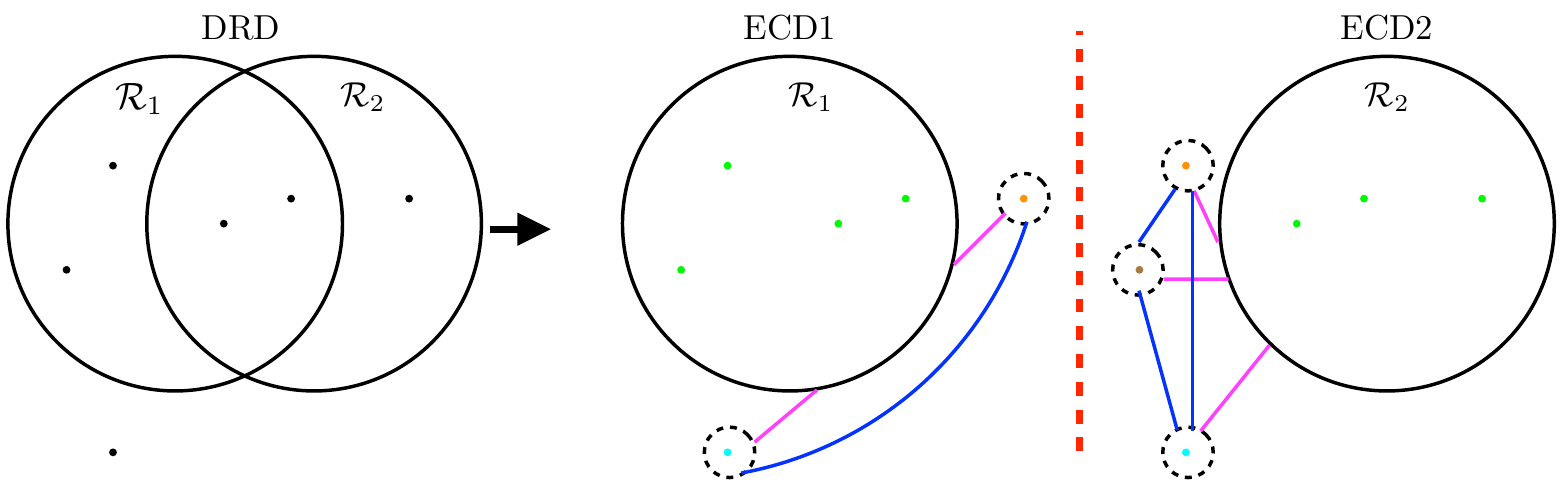}
    \caption{%
    \label{fig:direct_problem}
    The DRD problem split into `one region versus all' ECD problems by the \direct algorithm}
\end{figure}%

ECD problem $r$ creates a `one region versus all' problem using $\region_r$. The \ecsq objective corresponding to this problem is $\feci{r}{\obsOutcome}$. Note that $\feci{r}{\emptyset} = 0$ which corresponds to nothing. On the other hand $\feci{r}{\groundtruth} = 1$ which implies all edges are cut. The \direct algorithm then combines them in a \emph{Noisy-OR} formulation by defining the following combined objective

\begin{equation}
\label{eq:drd_fn}
  \fdrd{\obsOutcome} = 1 - \prod\limits_{r=1}^\numRegion (1 - \feci{r}{\obsOutcome})
\end{equation}

Note that $\fdrd{\obsOutcome} = 1$ iff $\feci{r}{\obsOutcome} = 1$ for at least one $r$. Thus the original DRD problem (\ref{eq:drd}) is equivalent to solving 

\begin{equation}
\label{eq:direct_drd}
\policyOpt \in \underbrace{\argminprob{\policy} \;\cost(\policy)}_\text{find policy} \; \mathrm{s.t} \; \underbrace{ \forall \groundtruth }_\text{groundtruth} \; : \;  \underbrace{\fdrd{\obsOutcomeFunc{\policy}{\groundtruth}} \geq 1}_{\text{drive the objective to 1}}
\end{equation}

\direct greedily selects a test $\test^* \in \argmaxprob{\test} \frac{\gain{\drd}{\test}{\obsOutcome}}{\cost(\test)}$.

\vspace{-0.5em}\subsubsection{The \algName algorithm}\vspace{-0.5em}

We can now evaluate the DRD objective in (\ref{eq:drd_fn}) using (\ref{eq:fec_applied})

\begin{equation}
\begin{aligned}
\label{eq:fdrd_applied}
&\fdrd{\obsOutcome}\\
&= 1 - \prod\limits_{r=1}^\numRegion (1 - \feci{r}{\obsOutcome}) \\
&= 1 - \prod\limits_{r=1}^\numRegion \left( 1 - 1 + \frac{\left(1 - \prod\limits_{i \in (\region_r \cap \selTestSet)} \Ind(\outcomeVarTest{i} = 1)
      \prod\limits_{j \in (\region_r \setminus \selTestSet)} \biasTest{j} \right)
      \left( \prod\limits_{k \in \region_r \cap \selTestSet} \biasTest{k}^{\obsOutcome(k)} (1 - \biasTest{k})^{1 - \obsOutcome(k)} \right)^2
}{1 - \prod\limits_{i \in \region_r} \biasTest{i}} \right) \\
&= 1 - \prod\limits_{r=1}^\numRegion \left(  \frac{\left(1 - \prod\limits_{i \in (\region_r \cap \selTestSet)} \Ind(\outcomeVarTest{i} = 1)
      \prod\limits_{j \in (\region_r \setminus \selTestSet)} \biasTest{j} \right)
      \left( \prod\limits_{k \in \region_r \cap \selTestSet} \biasTest{k}^{\obsOutcome(k)} (1 - \biasTest{k})^{1 - \obsOutcome(k)} \right)^2
}{1 - \prod\limits_{i \in \region_r} \biasTest{i}}  \right)
\end{aligned}
\end{equation}

\begin{lemma} \label{lem:drd}
The expression $\fdrd{\obsOutcome}$ is strongly adaptive monotone and adaptive submodular.
\end{lemma}
\begin{proof}
See Appendix \ref{sec:proof:lem_drd}
\end{proof}

\begin{theorem}
\label{eq:drd_near_opt}
Let $\numRegion$ be the number of regions, $\pminH$ the minimum prior probability of any hypothesis, $\policy_{DRD}$ be the greedy policy and $\policyOpt$ with the optimal policy. Then $\cost(\policy_{DRD}) \leq \cost(\policy^*)(2\numRegion \log \frac{1}{\pminH} + 1)$.
\end{theorem}
\begin{proof}
See Appendix \ref{proof:bisect_nearopt}
\end{proof}

\IncMargin{1em}
\begin{algorithm}[t]
  \caption{ Decision Region Determination with Independent Bernoulli Test$\left( \set{\region_i}_{i=1}^{\numRegion}, \biasVec, \groundtruth \right)$ }\label{alg:drd_skeleton}
  \algorithmStyle
  $\selTestSet \gets \emptyset$ \;
  \While{$(\nexists \region_i, P(\region_i | \obsOutcome) = 1) $ \textbf{\upshape and} $(\exists \region_i, P(\region_i | \obsOutcome) > 0) $}
  {
    $\candTestSet \gets \texttt{SelectCandTestSet}(\obsOutcome)$ \Comment*[r]{Using either (\ref{eq:cand_test_set:all}) or (\ref{eq:cand_test_set:maxp})}
    $\optTest \gets \texttt{SelectTest}(\candTestSet, \biasVec, \obsOutcome)$ \Comment*[r]{Using either (\ref{eq:greedy_fdrd}),(\ref{eq:policy_random}),(\ref{eq:policy_max_tally}),(\ref{eq:policy_set_cover}) or (\ref{eq:policy_mvoi})}  
    $\selTestSet \gets \selTestSet \cup \optTest$\; 
    $\outcomeTest{\optTest} \gets \groundtruth(\optTest)$ \Comment*[r]{Observe outcome for selected test}
  }
\end{algorithm}

We now describe the algorithm \algName. Algorithm \ref{alg:drd_skeleton} shows the framework for a general decision region determination algorithm. In order to specify \algName, we need to define two options - a candidate test set selection function $\texttt{SelectCandTestSet}(\obsOutcome)$ and a test selection function $\texttt{SelectTest}(\candTestSet, \biasVec, \obsOutcome)$.

The vanilla version of \algName implements $\texttt{SelectCandTestSet}(\obsOutcome)$ to return the set of all candidate tests $\candTestSet$ that contains only tests belonging to active regions that have not already been evaluated

\begin{equation}
  \label{eq:cand_test_set:all}
  \candTestSet = \set{\bigcup\limits_{i=1}^\numRegion \setst{\region_i}{P(\region_i | \obsOutcome) > 0}}
  \setminus \selTestSet
\end{equation}

We now examine the \algName test selection rule $\texttt{SelectTest}(\candTestSet, \biasVec, \obsOutcome)$ which can be simplified as  
\begin{equation}
\begin{aligned}
\label{eq:greedy_fdrd}
\test^*   &\in  \argmaxprob{\test \in \candTestSet}\;  \frac{ \gain{\drd}{\test}{\obsOutcome} }{\cost(\test)} \\
      &\in  \argmaxprob{\test \in \candTestSet}\; \frac{ \expect{\outcomeTest{\test}}{ \fdrd{ \obsOutcomeAdd{\test} } - \fdrd{\obsOutcome} \;|\; \obsOutcome } }{\cost(\test)}  \\
      &\in  \argmaxprob{\test \in \candTestSet}\; \frac{1}{\cost(\test)} \mathbb{E}_{\outcomeTest{\test}} \left[ 
       \prod\limits_{r=1}^\numRegion  
      \left(1 - \prod\limits_{i \in (\region_r \cap \selTestSet)} \Ind(\outcomeVarTest{i} = 1) \prod\limits_{j \in (\region_r \setminus \selTestSet)} \biasTest{j} \right) \right.\\
      & - \left. \left( \prod\limits_{r=1}^\numRegion 
      \left(1 - \prod\limits_{i \in (\region_r \cap \selTestSet \cup \test)} \Ind(\outcomeVarTest{i} = 1) \prod\limits_{j \in (\region_r \setminus \selTestSet \cup \test)} \biasTest{j} \right) \right)
      ( \biasTest{t}^{\outcomeTest{\test}} (1-\biasTest{t})^{1-\outcomeTest{\test}} )^{2\sum\limits_{k=1}^{m} \Ind(\test \in \region_k)} \right]
\end{aligned}
\end{equation}

Fig.~\ref{fig:canonical_drd_example} illustrates how \algName chooses different tests dependent on the bias vector $\biasVec$.

\begin{figure}[t]
    \centering
    \includegraphics[page=1,width=\textwidth]{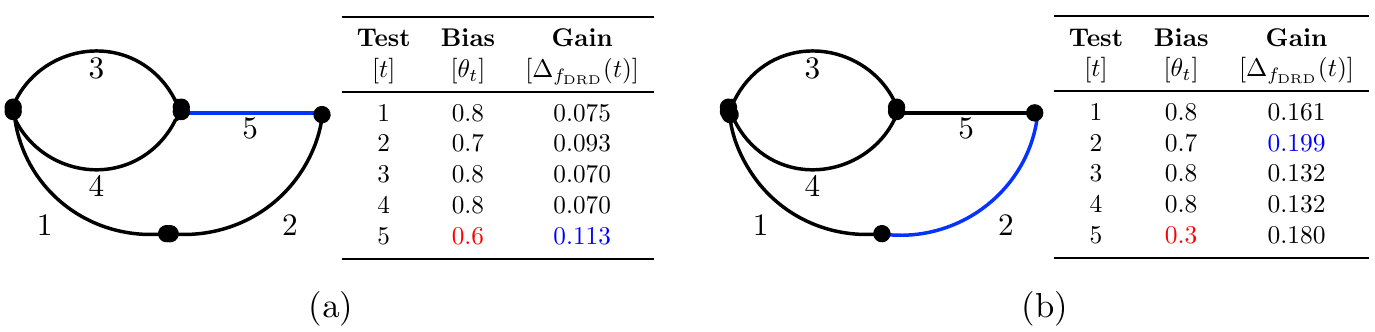}
    \caption{%
    \label{fig:canonical_drd_example}
    Canonical example illustrating \algName. In both scenarios, the paths remain the same but the bias vector $\biasVec$ varies. (a) Test $5$, which is common to $2$ paths. $\bias_5 = 0.6$ implies that $5$ is an informative test as its outcome not only affects the probability of a lot of paths, but it also has a slight likelihood of being collision free. Hence its gain is $0.113$. (b)  Setting $\bias_5 = 0.3$ reduces the likelihood of the test being true. Hence its no longer informative and instead test $2$ with gain $0.199$ is chosen.}
\end{figure}%

We now discuss the complexity of computing the marginal gain at each iteration. We have to cycle through $\numTest$ tests. For each tests, we only have to cycle through regions which it impacts. Let $\eta$ be the maximum number of regions that any test belongs to. For every region, we need to do an $O(1)$ operation of calculating the change in probability. Hence the complexity is $O(\numTest \eta)$. Note that this can be faster in practice by leveraging lazy methods in adaptive submodular problems (\citet{golovin2011adaptive}).

\subsection{Adaptively constraining test selection to most likely region}

We observe in our experiments that the surrogate (\ref{eq:fdrd_applied}) suffers from a slow convergence problem - $\fdrd{\obsOutcome}$ takes a long time to converge to $1$ when greedily optimized. This can be attributed to the curvature of the function.  To alleviate the convergence problem, we introduce an alternate candidate selection function $\texttt{SelectCandTestSet}(\obsOutcome)$ that assigns to $\candTestSet$ the set of all tests that belong to the most likely region $\maxProbTestSet$. We hence forth denote the constraint as \algMaxProbReg. It is evaluated as follows
\begin{equation}
  \label{eq:cand_test_set:maxp}
  \maxProbTestSet = \set{\argmaxprob{\region_i = \seq{\region}{\numRegion}}  \;P(\region_i | \obsOutcome) }
  \setminus \selTestSet
\end{equation}

Applying the constraint in (\ref{eq:cand_test_set:maxp}) leads to a dramatic improvement for any test selection policy as we will show in Sec.~\ref{sec:experiments:discussion}. The following theorem offers a partial explanation
\begin{theorem} \label{thm:max_prob}
A policy that greedily latches to a region according the the posterior conditioned on the region outcomes has a near-optimality guarantee of 4 w.r.t the optimal region evaluation sequence.
\end{theorem}
\begin{proof}
See Appendix \ref{sec:proof:max_prob}
\end{proof}

Applying the constraint in (\ref{eq:cand_test_set:maxp}) implies we are no longer greedily optimizing $\fdrd{\obsOutcome}$. However, the following theorem bounds the sub-optimality of this policy. 
\begin{theorem} \label{thm:sub_opt}
  Let $\pmin = \min_i P(\region_i)$, $\pminH = \min_{\hyp \in \hypSpace} P(\hyp)$ and $l = \max_i \abs{\region_i}$. The policy using (\ref{eq:cand_test_set:maxp}) has a suboptimality of $\alpha \left(2 \numRegion \log \left( \frac{1}{\pminH} \right) + 1 \right)$ where 
$\alpha \leq \left( 1 -   \max \left( (1 - \pmin)^2, \pmin^{\frac{2}{l}} \right) \right)^{-1}$.
\end{theorem} 
\begin{proof}
See Appendix \ref{sec:proof:sub_opt}
\end{proof}

The complexity of \algName with \algMaxProbReg reduces since we only have to visit states belonging to the most probable path. Finding the most probable path is an $O(\numRegion)$ operation. Let $l$ be the maximum number of tests in a region. Hence the complexity of gain calculation is $O(l \eta)$. The total complexity is $O(l \eta + \numRegion)$.

\section{Heuristic approaches to solving Bernoulli DRD problem}
We propose a collection of competitive heuristics that can also be used to solve the \probBernDRD problem. These heuristics are various $\texttt{SelectTest}(\candTestSet, \biasVec, \obsOutcome)$ policies in the framework of Alg.~\ref{alg:drd_skeleton}. To simplify the setting, we assume unit cost $\cost(\test) = 1$ although it would be possible to extend these to nonuniform setting. We also state the complexity for each algorithm and summarize them in Table~\ref{tab:complexity}.

\subsection{\algRandom}
The first heuristic \algRandom selects a test by sampling uniform randomly
\begin{equation}
  \label{eq:policy_random}
  \optTest \in \candTestSet
\end{equation}
The complexity is $O(1)$.

\subsection{\algMaxTally}
We adopt our next heuristic \algMaxTally from \citet{dellin2016unifying} by where the test belonging to most regions is selected. This criteria exhibits a `fail-fast' characteristic where the algorithm is incentivized to eliminate options quickly. This policy is likely to do well where regions have large amounts of overlap on tests that are likely to be in collision. 

\begin{equation}
  \label{eq:policy_max_tally}
  \optTest   \in  \argmaxprob{\test \in \candTestSet}\; \sum\limits_{i=1}^{\numRegion} \Ind\left(\test \in \region_i, P(\region_i | \obsOutcome) > 0 \right)
\end{equation}

To evaluate the complexity, we first describe how to efficiently implement this algorithm. Note that we can pre-process regions and tests to create a tally count of tests belonging to regions and a reverse lookup from tests to regions. Hence selecting a tests is simply finding the test with the max tally which is $O(\numTest)$. If the test is in collision, the tally count is updated by looking at all regions the test affects, and visiting tests contained by those regions to reduce their tally count. Let $\eta$ be the maximum regions to which a test belongs, and $l$ be the maximum number of tests contained by a region. Hence the complexity is $O(\numTest + \eta l)$. In the \algMaxProbReg setting, the complexity reduces to $O(l + \eta l) = O((1 + \eta)l)$.

\subsection{\algSetCover}
The next policy \algSetCover selects tests that maximize the expected number of `covered' tests, i.e. if a test is in collision, how many more tests are eliminated.
\begin{equation}
  \label{eq:policy_set_cover}
  \optTest \in  \argmaxprob{\test \in \candTestSet}\; 
  (1 - \biasTest{\test})
  \abs{ 
  \set{\bigcup\limits_{i=1}^\numRegion \setst{\region_i}{P(\region_i | \obsOutcome) > 0} - 
  \bigcup\limits_{j=1}^\numRegion \setst{\region_j}{P(\region_j | , \substack{\obsOutcome, \\ \outcomeVarTest{\test} = 0} ) > 0}   }
  \setminus \set{\selTestSet \cup \set{\test}}
  } 
\end{equation}

The motivation for this policy has its roots in the question - what is the optimal policy for checking \emph{all} paths? While \probBernDRD requires identifying one feasible region, it might still benefit from such a policy in situations where only one region is feasible. The following theorem states that greedily selecting tests according to the criteria above has strong guarantees. 

\begin{theorem}\label{thm:set_cover}
\algSetCover is a near-optimal policy for the problem of optimally checking all regions. 
\end{theorem}
\begin{proof}
See Appendix \ref{sec:proof:set_cover}
\end{proof}
We now analyze the complexity. We have to visit every test. Given a test is in collision, we have to compute the number of tests in the remaining regions which are not invalid. This would require visiting every test in every region. Hence the complexity is  $O(\numTest^2 \numRegion)$. In the \algMaxProbReg setting, the complexity reduces to $O(l \numRegion \numTest)$, where $l$ is the maximum number of tests contained by a region.

\subsection{\algMVOI}

The last baseline is a classic heuristic from decision theory: myopic value of information \citet{howard1966information}. 
We define a utility function $U(\hyp, \regionH) $ which is $1$ if $\hyp \in \regionH$ and $0$ otherwise. The utility of $\hypSpaceR(\obsOutcome)$ corresponds to the maximum expected utility of any decision region, i.e., the expected utility if we made a decision now. \algMVOI greedily chooses the test that maximizes (in expectation over observations) the utility as shown.

\begin{equation}
  \label{eq:policy_mvoi}
  \optTest   \in  \argmaxprob{\test \in \maxProbTestSet}\; (1 - \biasTest{\test}) \max\limits_{i = 1, \dots, \numRegion} P(\region_i \;|\; \obsOutcome, \outcomeVarTest{\test} = 0)
\end{equation}

Note that this test selection works only in the \algMaxProbReg setting. For every test in the most probable region, we eliminate regions that would invalid if the test is invalid. Let $l$ be the maximum number of tests contained by a region. Let $\eta$ be the maximum number of regions contained by a test. Then the complexity is $O(\eta l)$. 

\begin{table}[t]
\centering
\caption{Complexity of different algorithms (number of tests $\numTest$, number of regions $\numRegion$, maximum tests in a region $l$ and maximum regions belonging to a test $\eta$ )}
\begin{tabulary}{\textwidth}{LCCCCC}\toprule
                  & {\bf \algMVOI}       & {\bf \algRandom}       & {\bf \algMaxTally}       & {\bf \algSetCover}           & {\bf \algName} \\ \midrule
 Unconstrained    &                      & $O(1)$                 & $O(\numTest + \eta l)$   & $O(\numTest^2 \numRegion)$   & $O(\numTest \eta)$              \\
  MaxProbReg      &   $O(\eta l)$        & $O(1)$                 & $O((1 + \eta)l)$         & $O(l \numRegion \numTest)$   & $O(\eta l + \numRegion)$                 \\ \bottomrule
\end{tabulary}
\label{tab:complexity}
\end{table}

\section{Experiments}
\label{sec:experiments}
We evaluate all algorithms on a spectrum of synthetic problems, motion planning problems and experimental data from an autonomous helicopter. We present details on each dataset - motivation, construction of regions and tests and analysis of results. Table~\ref{tab:benchmark_results} presents the performance of all algorithms on all datasets. It shows the normalized cost with respect to algorithm $\algName$ $\policy_{\drd}$, i.e. $\frac{\cost(\policy) - \cost(\policy_{\drd}) }{\cost(\policy_{\drd})}$. The $95\%$ confidence interval value is shown (as a large number of samples are required to drive down the variance).
Finally, in Section \ref{sec:experiments:discussion}, we present a set of overall hypothesis and discuss their validity.

\subsection{Dataset 1: Synthetic Bernoulli Test}
\label{sec:dataset:bern}
\subsubsection{Motivation}
These datasets are designed to check the general applicability of our algorithms on problems which do not arise from graphs. Hence regions and tests are randomly created with minimal constraints that ensure the problems are non-trivial. 

\subsubsection{Construction}
First, a boolean region to test allocation matrix $\regTest \in \{0,1\}^{\numRegion \times \numTest}$ is created where $\regTest(i,j) = 1$ implies whether test $j$ belongs to region $\region_i$. $\regTest$ is randomly allocated by ensuring that each region $\region_i$ contains a random subset of tests. The number of such tests $l_i$ varies with region and is randomly sampled uniformly from $[0.05 \numTest, 0.10 \numTest]$. The bias vector $\biasVec \in \real^{1 \times \numTest}$ is sampled uniformly randomly from $[0.1, 0.9]$. A set of $\dataTest$ problems are created by sampling a ground truth $\groundtruth$ from $\biasVec$, and ensuring that at least one region is valid in each problem.

We set $\numTest = 100$ and $\dataTest = 100$. We create $3$ datasets by varying the number of regions $\numRegion = \{100, 500, 1000\}$. This is to investigate the performance of algorithms as the overlap among regions increase. 

\subsubsection{Analysis}
Table~\ref{tab:benchmark_results} shows the results as regions are varied. Among the unconstrained algorithms, \algName outperforms all other algorithms substantially with the gap narrowing on the Large dataset. For the \algMaxProbReg versions, \algName remains competitive across all datasets. \algMVOI matches its performance, doing better on dataset Large ($\numRegion = 1000$). 
From these results, we conclude that the datasets favour myopic behaviour. The performance of \algMVOI increases monotonically with $\numRegion$. This can be attributed to the fact that as the number of probable regions increase, myopic policies tend to perform better.

\begin{figure}[t]
    \centering
    \includegraphics[page=1,width=\textwidth]{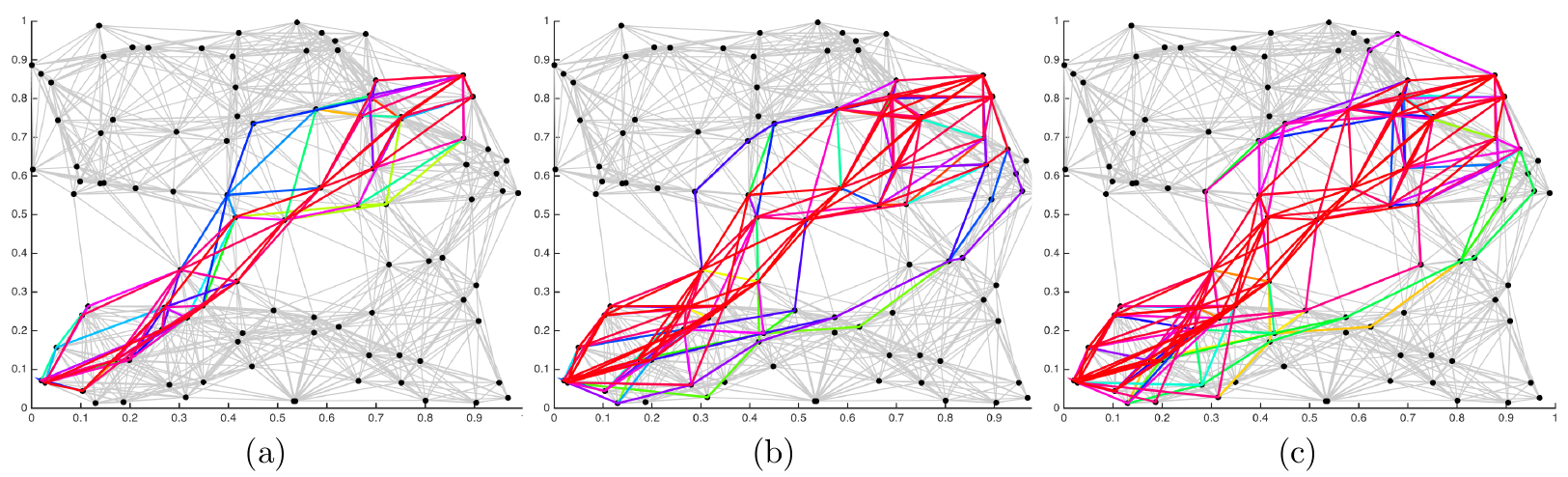}
    \caption{%
    \label{fig:synthetic_gbg}
    Construction of candidate path library $\PathSet$ for synthetic GBG experiments. The paths are embedded in an underlying RGG. (a) 100 paths (b) 500 paths (c) 1000 paths 
    }
\end{figure}%

\subsection{Dataset 2: Synthetic Generalized Binomial Graph}
\label{sec:dataset:gbg}
\subsubsection{Motivation}
These datasets are designed to test algorithms on GBG which do not necessarily arise out of motion planning problems.
For these datasets, edge independence is directly enforced. 
Difference between results on these datasets and those from motion planning can be attributed to spatial distribution of obstacles and overlap among regions.  

\subsubsection{Construction}
A randomg geometric graph (RGG) \cite{penrose2003random}  $\explicitGraph=(\vertexSet, \edgeSet)$ with $100$ vertices is sampled in a unit box $[0,1]\times[0,1]$. We create a set of paths $\PathSet$ from this graph by solving a set of shortest path problems (SPP). In each iteration of this algorithm, edges from $\explicitGraph$ are randomly removed with probability $0.5$ and the SPP is solved to produce $\Path$. This path is then appended to $\PathSet$ (if already not in the set) until $\abs{\PathSet} = \numRegion$. A bias vector $\biasVec \in \real^{1 \times \abs{\edgeSet}}$ is sampled uniformly randomly from $[0.1, 0.9]$. 

We create $3$ datasets by varying the number of paths $\numRegion = \{100, 500, 1000\}$. For each dataset, we create $\dataTest = 100$ problems. Fig.~\ref{fig:synthetic_gbg} shows the paths for these datasets. 

\subsubsection{Analysis}
Table~\ref{tab:benchmark_results} shows the results as the number of paths is increased. Among the unconstrained algorithms, \algMaxTally does better than \algName when $\numRegion$ is small. As $\numRegion$ increases, \algName outperforms all others and even matches up to its \algMaxProbReg version. This can be attributed to the fact that when $\numRegion$ is small, most of the paths pass through `bottleneck edges'. \algMaxTally inspects these edges first and if they are in collision, eliminates options quickly. As $\numRegion$ increases, the fraction of overlap decreases and problems become harder. For these problems, simply checking the most common edge does not suffice.

For the \algMaxProbReg version, we see that \algMaxTally has better overall performance. Thus we conclude that the combination of checking the most common edge and constraining to the most probable path works well. The difference between these datasets and Section~\ref{sec:dataset:bern} is that the region test allocation appears naturally from the graph structure. This leads to problems where `bottleneck edges' exist and \algMaxTally is able to identify them. Its interesting to note that \algMVOI performs worse as $\numRegion$ increases. This is because of the optimistic nature of \algMVOI - its less likely to select an edge that eliminates a lot of high probability regions (contrary to \algMaxTally). Hence the contrast between the two algorithms is displayed here.
Fig.~\ref{fig:illus_syn_gbg} shows an illustration of \algName selecting edges to solve a problem for $\numRegion = 50$.

\begin{figure}[t]
    \centering
    \includegraphics[page=1,width=\textwidth]{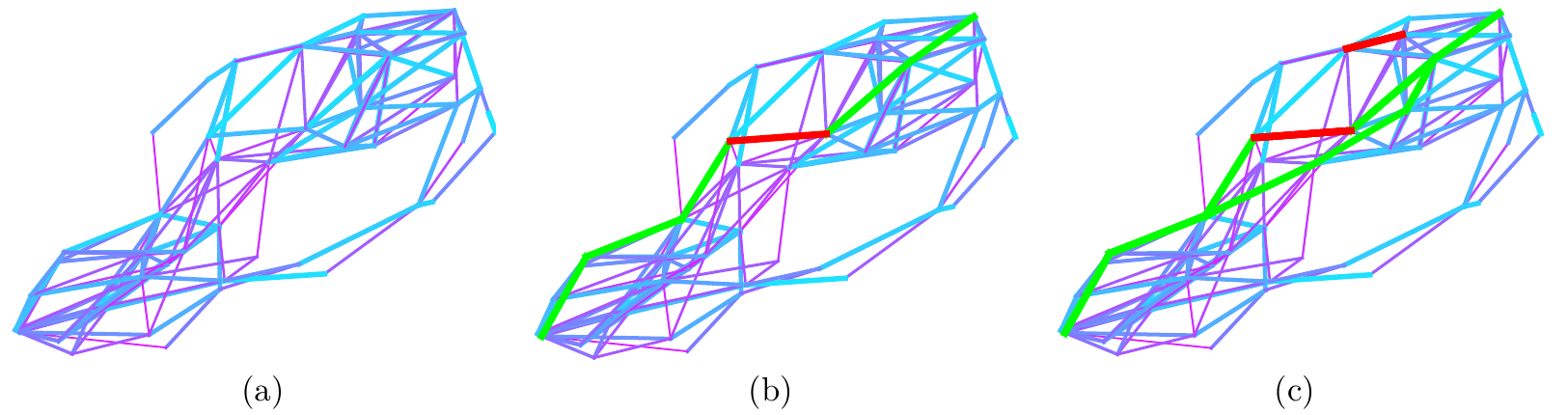}
    \caption{%
    \label{fig:illus_syn_gbg}
    Example of \algName as applied to a synthetic GBG problem. The GBG is shown with edges colored from magenta (thin) to cyan (thick) according to the prior likelihoods. (a) Initial state of the problem (b) \algName selects the most probable path and checks all its edges till it encounters and edge in collision in the middle of the path (c) It then looks at alternates till it discovers a valid short cut to connect the first and second half of the path}
\end{figure}%

\subsection{Dataset 3: 2D Geometric Planning} 
\label{sec:dataset:2dgeom}
\subsubsection{Motivation}
The main motivation for our work is robotic motion planning. The simplest instantiation is 2D geometric planning. The objective is to plan on a purely geometric graph where edges are invalidated by obstacles in the environment. Hence the probability of collision appears from the chosen distribution of obstacles. While the independent Bernoulli assumption is not valid, we will see that the algorithms still leverage such a prior to make effective decisions.

\subsubsection{Construction}
A random geometric graph (RGG) \cite{penrose2003random} $\explicitGraph = (\vertexSet, \edgeSet)$ with $\abs{\vertexSet} = 200$ is sampled in a unit box $[0,1]\times[0,1]$. We define a world map $\mathcal{M}$ as a binary map of occupied and unoccupied cells. Given $\explicitGraph$ and a $\mathcal{M}$, and edge $\edge \in \edgeSet$ is said to be in collision if it passes through an unoccupied cell. Fig.~\ref{fig:illus_explicit_graph}(a) shows an example of a collision checked RGG.
A parametric distribution can be used to create a distribution over world maps $P(\mathcal{M})$ which defines different environments. $P(\mathcal{M})$ can be used to measure the probability of individual edges being in collision. 

\begin{figure}[t]
    \centering
    \includegraphics[page=1,width=\textwidth]{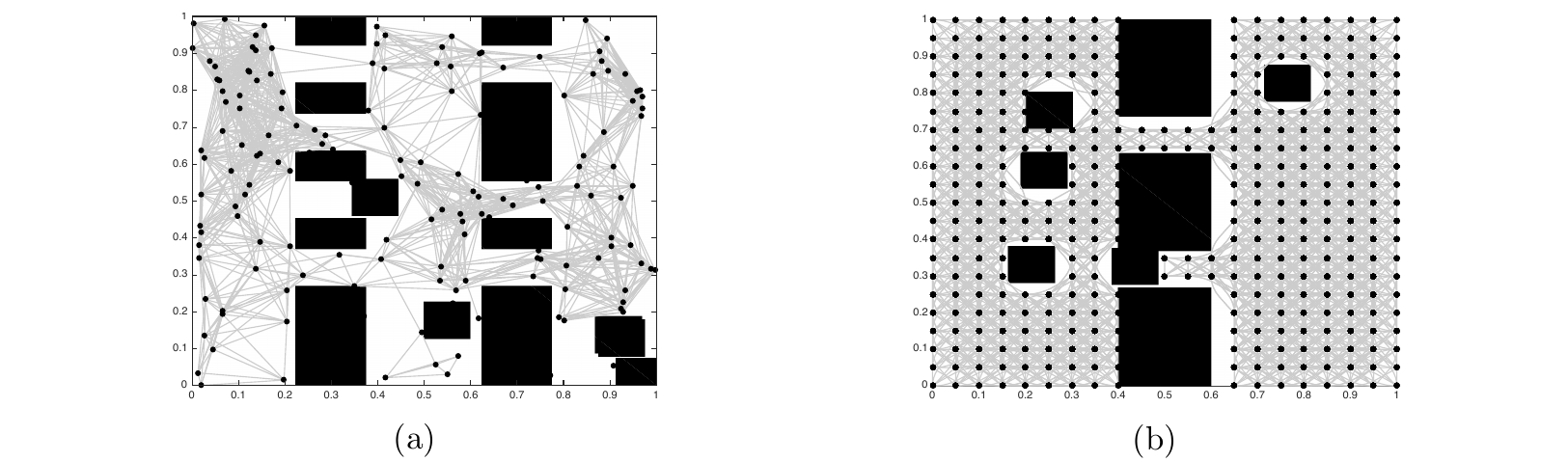}
    \caption{%
    \label{fig:illus_explicit_graph}
   Different explicit graphs for different problem settings (a) A RGG for 2D geometric planning (b) A state lattice for non-holonomic planning.
    }
\end{figure}%

\begin{figure}[t]
    \centering
    \includegraphics[page=1,width=\textwidth]{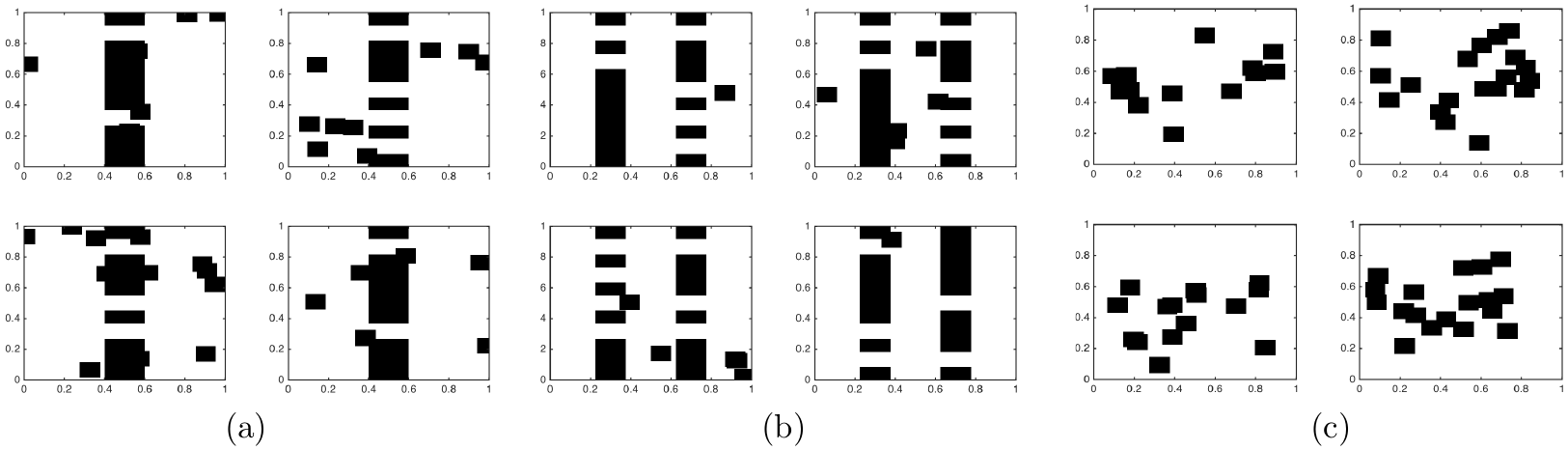}
    \caption{%
    \label{fig:2d_world_maps}
    Different datasets of environments (a) OneWall (b) TwoWall (c) Forest
    }
\end{figure}%

We create $3$ datasets corresponding to different environments as shown in Fig.~\ref{fig:2d_world_maps} - Forest, OneWall, TwoWall. These datasets are created by defining parametric distributions that distribute rectangular blocks. Forest corresponds to a non uniform stationary distribution of squares to mimick a forest like environment where trees are clustered together with spatial correlations. OneWall is created by constructing a wall with random gaps in conjunction with a uniform random distribution of squares. TwoWall contains two such walls. Hence these datasets create a spectrum of difficulty to test our algorithms. 

We now describe the method for constructing the set of paths $\PathSet$. We would like a set of good candidate paths on the distribution $P(\mathcal{M})$. We define a goodness function as the probability of atleast one path in the set to be valid on the dataset. Following the methodology in \citet{tallavajhula2016list}, we use a greedy method. We sample a training dataset consisting of $\dataTrain = 1000$ problems. On every problem in this dataset, we solve the shortest path problem to get a path $\Path$. We then greedily construct $\PathSet$ by selecting the path that is most valid till our budget $\numRegion$ is filled. We set $\numRegion = 500$ for all datasets.

\subsubsection{Analysis}
Table~\ref{tab:benchmark_results} shows the results on all 3 datasets. In the unconstrained case, \algName outperforms all other algorithms by a significant margin. For the \algMaxProbReg version, \algName remains competitive. The closest competitor to it is \algSetCover - matching performance in the TwoWall dataset. Further analysis of this dataset revealed that the dataset has problems that are difficult - where only one of the paths in the set are feasible. This often requires eliminating all other paths. \algSetCover performs well under such situations due to guarantees described in Theorem \ref{thm:set_cover}. 

These results vary from the patterns in Section \ref{sec:dataset:gbg}. This is to do with the relationship with overlap of regions and priors on tests. Since the regions are created in a way cognizant of the prior, regions often overlap on tests that are likely to be free with high probability. \algMaxTally ignores this bias term and hence prioritizes checking such edges first even if they offer no information. 

Table~\ref{tab:benchmark_results} also shows results on varying the number of regions. \algName is robust to this change. \algSetCover performs better with less number of paths. This can be attributed to the path that the number of feasible path decreases, thus becoming advantageous to check all paths. 

Fig.~\ref{fig:2d_env_comparison} shows a comparison of all algorithms on a problem from OneWall dataset. It illustrates the contrasting behaviours of all algorithms. \algMaxTally selects edges belonging to many paths which happens to be near the start / goal. These are less likely to be discriminatory. \algSetCover takes time to converge as it attempts to cover all edges. \algMVOI focuses on edges likely to invalidate the current most probable path which eliminates paths myopically but takes time to converge. \algName enjoys the best of all worlds. 

\begin{figure}[t]
    \centering
    \includegraphics[page=1,width=\textwidth]{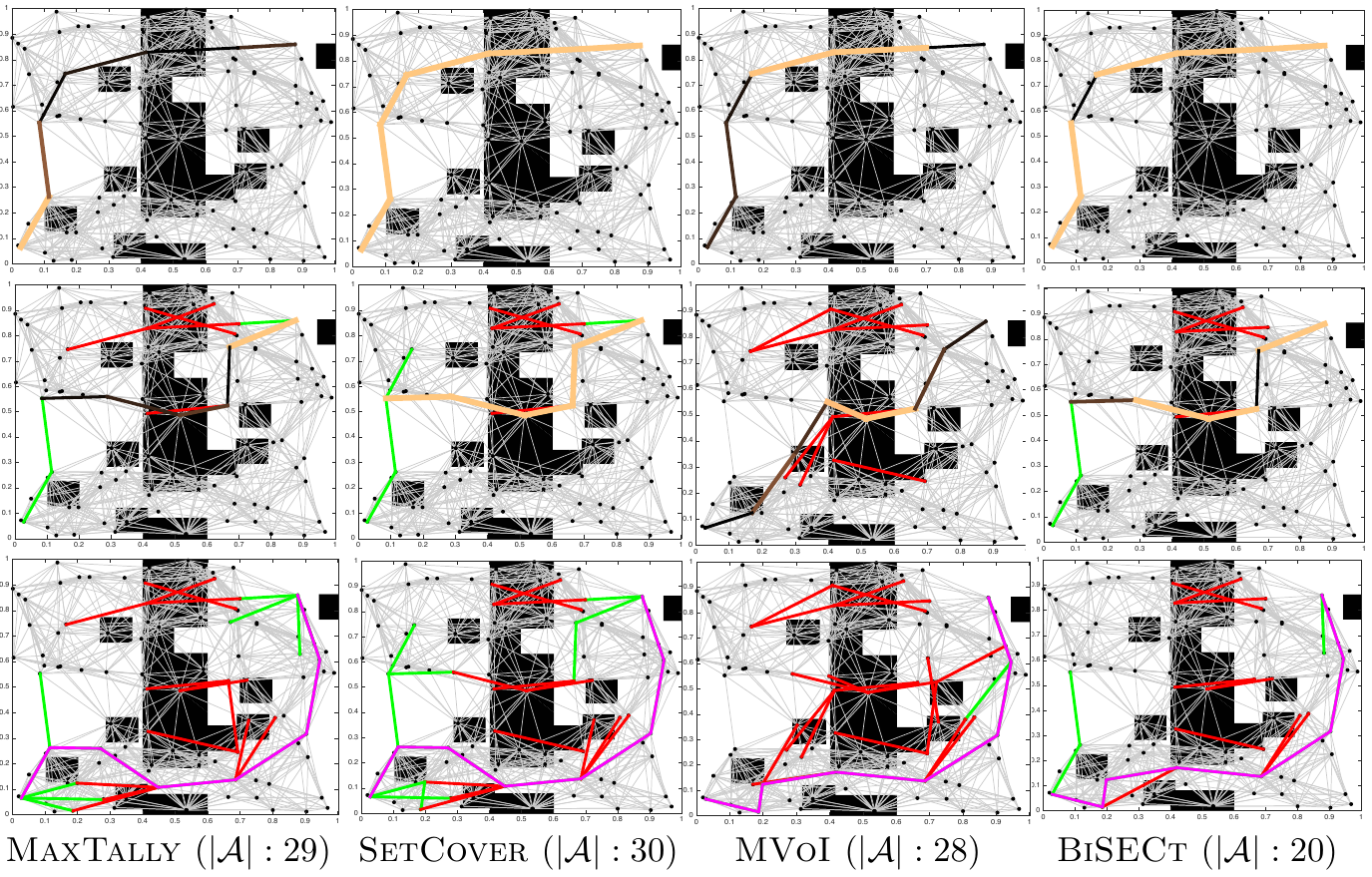}
    \caption{%
    \label{fig:2d_env_comparison}
    Performance (number of evaluated edges) of all algorithms on 2D geometric planning. Snapshots are at interim and final stages respectively show evaluated valid edges (green), invalid (red) edges and final path (magenta). The marginal gain of candidate edges goes from black (low) to cream (high).}
\end{figure}%

\subsection{Dataset 4: Non-holonomic Path Planning}

\subsubsection{Motivation}
While 2D geometric planning examined the influence of various spatial distribution of obstacles on random graphs, it does not impose a constraint on the class of graphs. Hence we look at the more practical case of mobile robots with constrained dynamics. This robots plan on a state-lattice (\citet{pivtoraiko2009differentially}) - a graph where edges are dynamically feasible maneuvers. As motivated in Section ~\ref{sec:intro}, these problems are of great importance as a robot has to react fast to safely avoid obstacles. The presence of differential constraint reduces the set of feasible paths, hence requiring checks at a greater resolution. 

\subsubsection{Construction} 
The vehicle being considered is a planar curvature constrained system. Hence the search space is 3D - x, y and yaw. A state lattice of dynamically feasible maneuvers is created as shown in Fig.~\ref{fig:illus_explicit_graph}(b). The environments are used from Section ~\ref{sec:dataset:2dgeom} - Forest and OneWall. The density of obstacle in these datasets are altered to allow constrained system to find solutions. The candidate set of paths are created in a similar fashion as in Section.~\ref{sec:dataset:2dgeom}. We set $\numRegion \approx 100$ for all datasets.

\subsubsection{Analysis}
Table~\ref{tab:benchmark_results} shows results across datasets. In the unconstrained setting, \algName significantly outperforms other algorithms. In the \algMaxProbReg setting, we see that \algSetCover is equally competitive. The analysis of the Forest dataset reveals that due to the difficulty of the dataset, problems are such that only one of the paths is free. As explained in Section \ref{sec:dataset:2dgeom}, \algSetCover does well in such settings. On the OneWall dataset, we see several algorithms performing comparatively. This might indicate the relative easiness of the dataset. 

Table~\ref{tab:benchmark_results} shows variation across degree of the lattice. We see that \algName remains competitive across this variation. 

\subsection{Dataset 5: 7D Arm Planning}
\label{sec:dataset}
\subsubsection{Motivation}
An important application for efficient edge evaluations is planning for a 7D arm. Edge evaluation is expensive geometric intersection operations are required to be performed to ascertain validity. A detailed motivation is provided in \citet{dellin2016unifying}. Efficient collision checking would allow such systems to plan quickly while performing tasks such as picking up and placing objects from one tray to another. One can additionally assume an unknown agent present in the workspace. Such problems would benefit from reasoning using priors on edge validity. 

\subsubsection{Construction}
A random geometric graph with $7052$ vertices and $16643$ edges is created (as described in \citet{dellin2016unifying}). Edges in self-collision are prune apriori. We create $2$ datasets to simulate pick and place tasks in a kitchen like environment. The start and goal from all problem is from one end-effector position to another. The first dataset - Table - comprises simply of a table at random offsets from the robot. The location of the table invalidates large number of edges. The second dataset - Clutter - comprises of an object and table at random offsets from the robot. In all datasets, a random subset corresponding to $0.3$ fraction of free edges are `flipped', i.e. made to be in collision. This creates the effect of random disturbances in the environment. Paths are created in a similar way as Section \ref{sec:dataset:2dgeom}. 
We set $\numRegion \approx 200$ for all datasets. Fig.~\ref{fig:herb} shows an illustration of the problems. 

\begin{figure}[t]
    \centering
    \includegraphics[page=1,width=\textwidth]{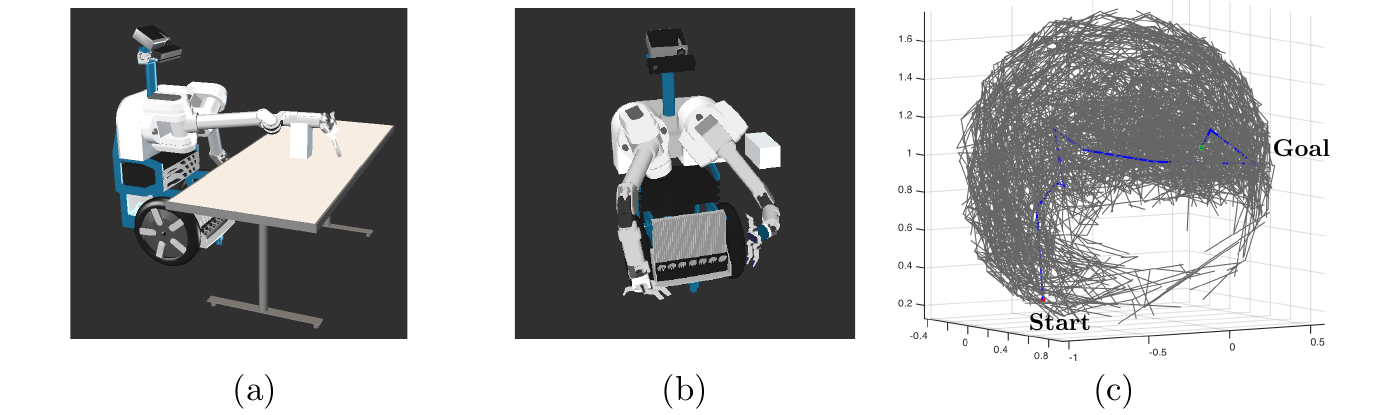}
    \caption{%
    \label{fig:herb}
    7D arm planning dataset (a) Snapshot of the manipulator for planning with a table (b) Snapshot of manipulator planning with an object (c) The explicit graph shown as straight line connections between end effector locations (also subsampled $50\%$). The start and goal end effector locations are also shown. Edges in collision are removed. 
    }
\end{figure}%

\subsubsection{Analysis}
Table~\ref{tab:benchmark_results} shows results across datasets. In both the unconstrained and \algMaxProbReg setting, \algName significantly outperforms other algorithms. \algMaxTally in the \algMaxProbReg is the next best performing policy. This suggests that the dataset might lead to bottleneck edges - edges through which many paths pass through that can be in collision. Further analysis reveals, this artifact occurs due to the random disturbance. \algMaxTally is able to verify quickly if such bottleneck edges are in collision, and if so remove a lot of candidate paths from consideration. 

\begin{figure}[t]
    \centering
    \includegraphics[page=1,width=\textwidth]{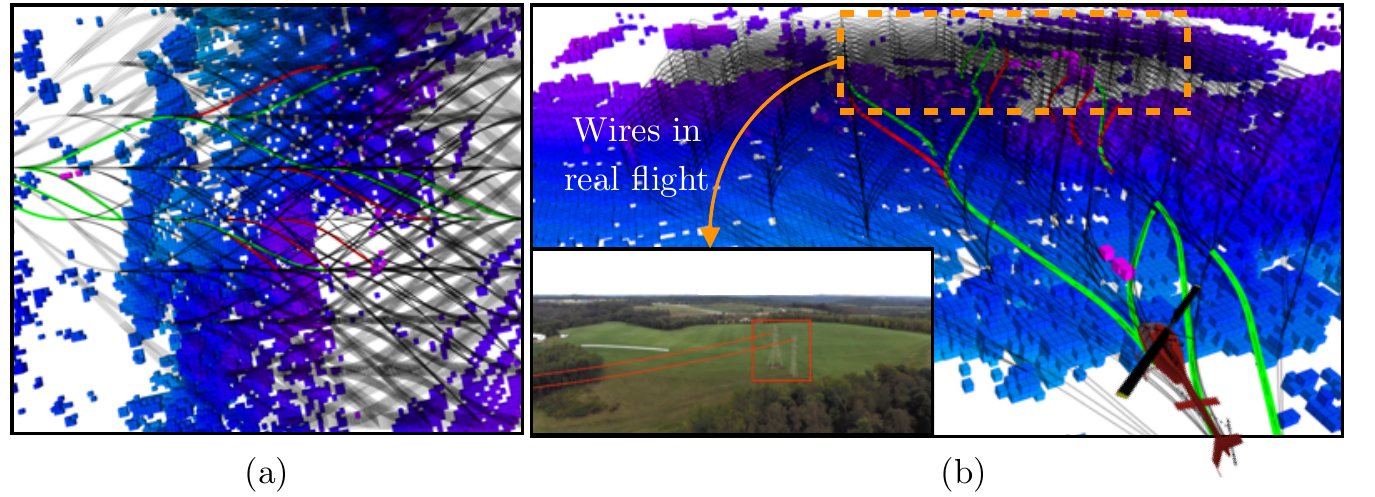}
    \caption{%
    \label{fig:heli}
    Experimental data from a full scale helicopter that has to avoid wires as it comes into land. The helicopter detects wires fairly late which requires an instant avoidance maneuver. The helicopter uses a state lattice and has to quickly identify a feasible path on the lattice. Evaluating edges are expensive since the system has to ensure it avoids wires by a sufficient clearance.
    (a) A top down view of the state lattice. Maneuvers are lateral as well as vertical.
    (b) Performance of \algName on the motion planning problem. The voxels in blue represent occupied locations in the world as detected by the helicopter. The wires (as seen in the camera) appear as a small set of voxels in the map. \algName selectively evaluates certain edges of the state lattice (green shows edges evaluated to be valid, red shows edges evaluated to be invalid). It is quickly able to identify a feasible path. 
    \fullFigGap}
\end{figure}%

\subsection{Autonomous Helicopter Wire Avoidance}
We now evaluate our algorithms on experimental data from a full scale helicopter. The helicopter is equipped with a laser scanner that scans the world to build a model of obstacles and free space. The system is required to plan around detected obstacles as it performs various missions.

A particularly difficult problem is dealing with wires as the system comes in to land. The system has limitations on how fast it can ascend / descend. Hence it has to not only react fast, but determine which direction to move so as to feasibly land. Fig.~\ref{fig:heli} shows the scenario. In this domain, edge evaluation is expensive because given an edge, it must be checked at a high resolution to ensure it is as sufficient distance from an obstacle. 

Fig.~\ref{fig:heli} (b) shos how \algName evaluates informative edges to identify a feasible path. This algorithm uses priors collected in simulation of wire like environments. 

\begin{table}[!htpb]
\small
\centering
\caption{Normalized cost of different algorithms on different datasets ($95\%$ C.I.) }
\begin{tabulary}{\textwidth}{LCCCCC}\toprule
       & {\bf \algMVOI}       & {\bf \algRandom}       & {\bf \algMaxTally}       & {\bf \algSetCover}    & {\bf \algName} \\ 
       &                      & {Unconstrained}    & {Unconstrained}   & {Unconstrained}   & {Unconstrained}               \\
       &                      & {MaxProbReg}        & {MaxProbReg}      & {MaxProbReg}      & {MaxProbReg}                 \\ \midrule
\multicolumn{6}{c}{ {\bf Synthetic Bernoulli Test: Variation across region overlap} }   \\
Small
            &                   & $(4.18, 6.67)$  & $(3.49, 5.23)$  & $(1.77, 3.01)$  & $(1.42, 2.36)$    \\
$\numRegion:100$
            & $(0.00, 0.08)$     & $(0.12, 0.29)$  & $(0.12, 0.25)$  & $(0.18, 0.40)$  & $\red(0.00, 0.00)$  \\ 
Medium
            &                   & $(3.27, 4.40)$  & $(3.04, 4.30)$  & $(3.55, 4.67)$  & $(1.77, 2.64)$   \\
$\numRegion:500$  
            & $\red(0.00, 0.00)$& $(0.05, 0.25)$  & $(0.14, 0.24)$  & $(0.14, 0.33)$  & $\red(0.00, 0.00)$  \\ 
Large
            &                   & $(2.86, 4.26)$  & $(2.62, 3.85)$  & $(2.94, 3.71)$  & $(1.33, 1.81)$   \\
$\numRegion:1000$
            & $\red(-0.11, 0.00)$& $(0.00, 0.28)$ & $(0.06, 0.26)$  & $(0.09, 0.22)$  & $(0.00, 0.00)$   \\
\multicolumn{6}{c}{ {\bf Synthetic Bernoulli Test: Variation across region overlap} }   \\
Small
            &                   & $(6.08, 7.25)$  & $(0.68, 1.50)$  & $(2.12, 2.50)$  & $(1.27, 1.50)$    \\
$\numRegion:100$
            & $(0.00, 0.00)$     & $(0.00, 0.00)$  & $\red(-0.13, -0.11)$  & $(0.13, 0.14)$  & $(0.00, 0.00)$  \\ 
Medium
            &                   & $(6.51, 8.53)$  & $(0.12, 0.51)$  & $(1.43, 1.75)$  & $(0.15, 0.46)$   \\
$\numRegion:500$  
            & $(0.00, 0.11)$    & $(0.00, 0.11)$  & $(0.00, 0.09)$  & $\red(-0.04, 0.07)$  & $\red(0.00, 0.00)$  \\ 
Large
            &                   & $(9.65, 11.67)$ & $(0.63, 1.18)$        & $(2.24, 2.89)$  & $(0.31, 0.63)$   \\
$\numRegion:1000$
            & $(0.13, 0.24)$    & $(0.00, 0.11)$  & $\red(-0.13, -0.07)$  & $(0.11, 0.13)$  & $(0.00, 0.00)$   \\
\multicolumn{6}{c}{ {\bf 2D Geometric Planning: Variation across environments} }   \\
Forest    &                   & $(19.45, 27.66)$  & $(4.68, 6.55)$  & $(3.53, 5.07)$   & $(1.90, 2.46)$    \\
          & $(0.03, 0.18)$    & $(0.13, 0.30)$    & $(0.09, 0.18)$  & $(0.00, 0.09)$   & $\red(0.00, 0.00)$   \\ 
OneWall   &                   & $(13.35, 17.79)$  & $(4.12, 4.89)$   & $(1.36, 2.11)$    & $(0.76, 1.20)$    \\ 
          & $(0.045, 0.21)$   & $(0.11, 0.42)$    & $(0.00, 0.12)$   & $(0.14, 0.29)$    & $\red(0.00, 0.00)$  \\ 
TwoWall   &                   & $(13.76, 16.61)$  & $(2.76, 3.93)$   & $(2.07, 2.94)$    & $(0.91, 1.44)$    \\
          & $(0.00, 0.09)$    & $(0.33, 0.51)$    & $(0.10, 0.20)$   & $\red(0.00, 0.00)$    & $\red(0.00, 0.00)$\\
\multicolumn{6}{c}{ {\bf 2D Geometric Planning: Variation across region size} }   \\
OneWall    &                  & $(12.06, 16.01)$  & $(4.47, 5.13)$  & $(2.00, 3.41)$   & $(0.94, 1.42)$    \\
$\numRegion:300$
          & $(0.00, 0.17)$    & $(0.12, 0.42)$    & $(0.06, 0.24)$  & $(0.00, 0.38)$   & $\red(0.00, 0.00)$  \\ 
OneWall   &                   & $(13.26, 16.79)$  & $(2.18, 3.77)$   & $(1.04, 1.62)$    & $(0.41, 0.91)$    \\ 
$\numRegion:858$ 
          & $(0.00, 0.14)$   & $(0.09, 0.27)$    & $\red(-0.04, 0.08)$   & $(0.00, 0.14)$    & $\red(0.00, 0.00)$  \\ 
\multicolumn{6}{c}{ {\bf Non-holonomic Path Planning: Variation across environments} }   \\
Forest    &                   & $(22.38, 29.67)$  & $(9.79, 11.14)$    & $(2.63, 5.28)$ & $(1.54, 2.46)$  \\
          & $(0.09, 0.18 )$   & $(0.46, 0.79)$    & $(0.25, 0.38)$     & $\red(0.00, 0.00)$ & $\red(0.00, 0.00)$\\ 
OneWall   &                   & $(13.02, 15.75)$  & $(8.40, 11.47)$    & $(3.72, 4.54)$   & $(3.28, 3.78)$    \\ 
          & $\red(-0.11, 0.11)$& $(0.00, 0.12)$    & $(0.21, 0.28)$    & $\red(-0.11, 0.11)$  & $\red(0.00, 0.00) $ \\ 
\multicolumn{6}{c}{ {\bf Non-holonomic Path Planning: Variation across lattice degree} }   \\
OneWall    &                    & $(10.46, 11.57)$  & $(3.95, 4.83)$  & $(0.83, 1.18)$   & $(0.24, 0.58)$    \\
$\degree:12$ & $(0.04, 0.11)$   & $(0.30, 0.56)$    & $(0.11, 0.18)$  & $(0.00, 0.06)$   & $\red(0.00, 0.00)$  \\ 
OneWall    &                    & $(14.97, 17.90)$  & $(9.19, 13.11)$ & $(3.22, 5.07)$   & $(2.16, 2.81)$    \\
$\degree:30$ & $(0.05, 0.10)$   & $(0.14, 0.40)$    & $(0.20, 0.52)$  & $(0.00, 0.03)$   & $\red(0.00, 0.00)$\\ 
\multicolumn{6}{c}{ {\bf 7D Arm Planning: Variation across environments} }   \\ 
Table     &                   & $(15.12, 19.41)$  & $(4.80, 6.98)$      & $(1.36, 2.17)$      & $(0.32, 0.67)$    \\ 
          & $(0.28, 0.54)$    & $(0.13, 0.31)$    & $(0.00, 0.04)$      & $(0.00, 0.11)$      & $\red(0.00, 0.00)$  \\ 
Clutter   &                   & $(7.92, 9.85)$    & $(3.96, 6.44)$      & $(1.42, 2.07)$      & $(1.23, 1.75)$    \\
          &$(0.02, 0.20)$     & $(0.14, 0.36)$    & $\red(0.00, 0.00)$  & $(0.00, 0.11)$      & $\red(0.00, 0.00)$ \\ \bottomrule
\end{tabulary}
\label{tab:benchmark_results}
\end{table}

\begin{figure}[t]
    \centering
    \includegraphics[page=1,width=\textwidth]{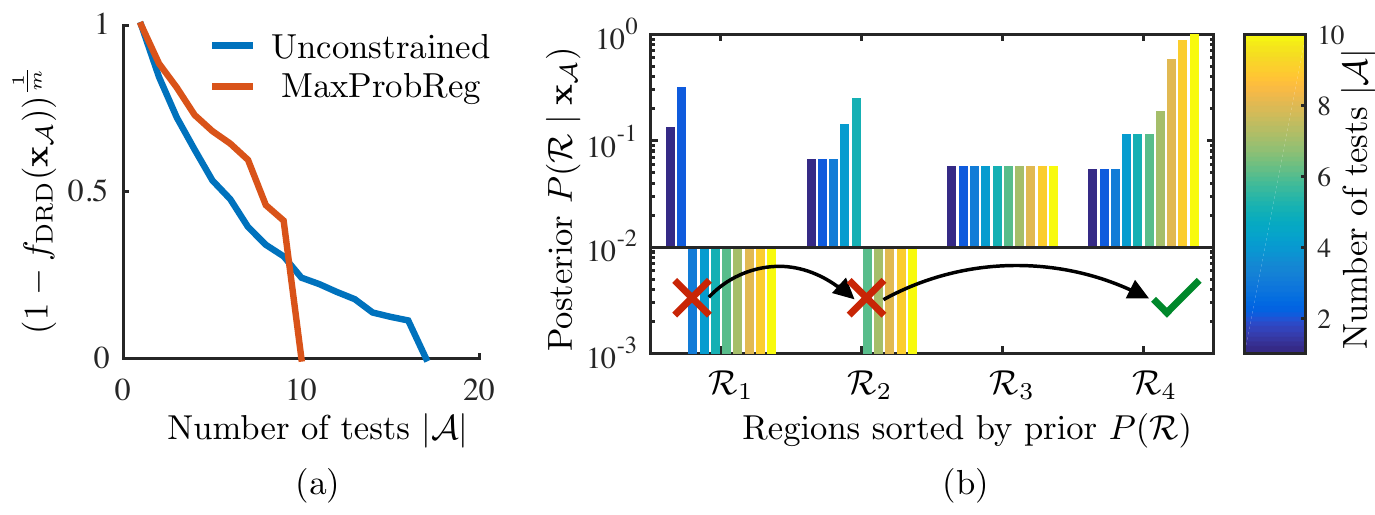}
    \caption{%
    \label{fig:histogram_max_prob}
    (a) Illustration of convergence issues for $\fdrd{\obsOutcome}$ - the transformation $(1-\fdrd{\obsOutcome})^{\frac{1}{\numRegion}}$ shows that it flattens out thus allowing even a non-greedy algorithm to converge faster. 
    (b) \algName with \algMaxProbReg shown in the space of posterior probabilities of region. First $\region_1$ is checked, then $\region_2$ and finally $\region_4$ is found to be valid. }
\end{figure}%

\subsection{Overall summary of results} 
\label{sec:experiments:discussion}
Table~\ref{tab:benchmark_results} shows the evaluation cost of all algorithms on various datasets normalized w.r.t \algName. The two numbers are lower and upper $95\%$ confidence intervals - hence it conveys how much fractionally poorer are algorithms w.r.t \algName. The best performance on each dataset is highlighted. We present a set of observations to interpret these results.
\begin{observation}
\algName has a consistently competitive performance across all datasets.
\end{observation} 

Table~\ref{tab:benchmark_results} shows on $13$ datasets, \algName is at par with the best - on $8$ of those it is exclusively the best. 

\begin{observation}
The \algMaxProbReg variant improves the performance of all algorithms on most datasets
\end{observation} 

Table~\ref{tab:benchmark_results} shows that this is true on $12$ datasets. The impact is greatest on \algRandom where improvement is upto a factor of $20$. For the case of \algName, Fig.~\ref{fig:histogram_max_prob}(a) illustrates the problem by examining the shape of $(1-\fdrd{\obsOutcome})^{\frac{1}{\numRegion}}$. Even though $\fdrd{\obsOutcome}$ is submodular, it flattens drastically allowing a non-greedy policy to converge faster. Fig.~\ref{fig:histogram_max_prob}(b) shows how the probability of region evolves as tests are checked in the \algMaxProbReg setting. We see this `latching' characteristic - where test selection drives a region probability to 1 instead of exploring other tests. 

However, this is not true in general. See Appendix \ref{app:max_prob_reg} for results on datasets with large disparity in region sizes.
\begin{observation}
On planning problems, \algName strikes a trade-off between the complimentary natures of \algMaxTally and \algMVOI.
\end{observation} 

We examine this in the context of 2D planning as shown in Fig.~\ref{fig:2d_env_comparison}. 
\algMaxTally selects edges belonging to many paths which is useful for path elimination but does not reason about the event when the edge is not in collision.
\algMVOI selects edges to eliminate the most probable path but does not reason about how many paths a single edge can eliminate. 
\algName switches between these behaviors thus achieving greater efficiency than both heuristics.

\begin{observation}
\algName checks informative edges in collision avoidance problems encountered a helicopter
\end{observation} 

Fig.~\ref{fig:heli}(b) shows the efficacy of \algName on experimental flight data from a helicopter avoiding wire.

\vspace{-0.7em}\section{Conclusion}\vspace{-0.7em}
\label{sec:conclusion}
In this paper, we addressed the problem of identification of a feasible path from a library while minimizing the expected cost of edge evaluation given priors on the likelihood of edge validity. We showed that this problem is equivalent to a decision region determination problem where the goal is to select tests (edges) that drive uncertainty into a single decision region (a valid path). 
We proposed \algName, and efficient and near-optimal algorithm that solves this problem by greedily optimizing a surrogate objective.
We validated \algName on a spectrum of problems against state of the art heuristics and showed that it has a consistent performance across datasets. 
This works serves as a first step towards importing Bayesian active learning approaches into the domain of motion planning. 

\bibliographystyle{plainnat}
\bibliography{reference}

\begin{thebibliography}{28}
\providecommand{\natexlab}[1]{#1}
\providecommand{\url}[1]{\texttt{#1}}
\expandafter\ifx\csname urlstyle\endcsname\relax
  \providecommand{\doi}[1]{doi: #1}\else
  \providecommand{\doi}{doi: \begingroup \urlstyle{rm}\Url}\fi

\bibitem[Bohlin and Kavraki(2000)]{bohlin2000path}
Robert Bohlin and Lydia~E Kavraki.
\newblock Path planning using lazy prm.
\newblock In \emph{ICRA}, 2000.

\bibitem[Burns and Brock(2005)]{burns2005sampling}
Brendan Burns and Oliver Brock.
\newblock Sampling-based motion planning using predictive models.
\newblock In \emph{ICRA}, 2005.

\bibitem[Chaloner and Verdinelli(1995)]{chaloner1995bayesian}
Kathryn Chaloner and Isabella Verdinelli.
\newblock Bayesian experimental design: A review.
\newblock \emph{Statistical Science}, pages 273--304, 1995.

\bibitem[Chen et~al.(2015)Chen, Javdani, Karbasi, Bagnell, Srinivasa, and
  Krause]{chen2015submodular}
Yuxin Chen, Shervin Javdani, Amin Karbasi, J.~Andrew~(Drew) Bagnell, Siddhartha
  Srinivasa, and Andreas Krause.
\newblock Submodular surrogates for value of information.
\newblock In \emph{AAAI}, 2015.

\bibitem[Choudhury et~al.(2016{\natexlab{a}})Choudhury, Gammell, Barfoot,
  Srinivasa, and Scherer]{Choudhury_2016_8070}
Sanjiban Choudhury, Jonathan~D. Gammell, Timothy~D. Barfoot, Siddhartha
  Srinivasa, and Sebastian Scherer.
\newblock Regionally accelerated batch informed trees (rabit*): A framework to
  integrate local information into optimal path planning.
\newblock In \emph{ICRA}, 2016{\natexlab{a}}.

\bibitem[Choudhury et~al.(2016{\natexlab{b}})Choudhury, Dellin, and
  Srinivasa]{choudhury2016pareto}
Shushman Choudhury, Christopher~M Dellin, and Siddhartha~S Srinivasa.
\newblock Pareto-optimal search over configuration space beliefs for anytime
  motion planning.
\newblock In \emph{IROS}, 2016{\natexlab{b}}.

\bibitem[Cohen et~al.(2015)Cohen, Phillips, and Likhachev]{cohen2015planning}
Benjamin Cohen, Mike Phillips, and Maxim Likhachev.
\newblock Planning single-arm manipulations with n-arm robots.
\newblock In \emph{Eigth Annual Symposium on Combinatorial Search}, 2015.

\bibitem[Cover et~al.(2013)Cover, Choudhury, Scherer, and
  Singh]{cover2013sparse}
Hugh Cover, Sanjiban Choudhury, Sebastian Scherer, and Sanjiv Singh.
\newblock Sparse tangential network (spartan): Motion planning for micro aerial
  vehicles.
\newblock In \emph{ICRA}. IEEE, 2013.

\bibitem[Dasgupta(2004)]{dasgupta2004analysis}
Sanjoy Dasgupta.
\newblock Analysis of a greedy active learning strategy.
\newblock In \emph{NIPS}, 2004.

\bibitem[Dellin and Srinivasa(2016)]{dellin2016unifying}
Christopher~M Dellin and Siddhartha~S Srinivasa.
\newblock A unifying formalism for shortest path problems with expensive edge
  evaluations via lazy best-first search over paths with edge selectors.
\newblock In \emph{ICAPS}, 2016.

\bibitem[Dellin et~al.(2016)Dellin, Strabala, Haynes, Stager, and
  Srinivasa]{dellin2016guided}
Christopher~M Dellin, Kyle Strabala, G~Clark Haynes, David Stager, and
  Siddhartha~S Srinivasa.
\newblock Guided manipulation planning at the darpa robotics challenge trials.
\newblock In \emph{Experimental Robotics}, 2016.

\bibitem[Dor(1998)]{dor1998greedy}
Avner Dor.
\newblock The greedy search algorithm on binary vectors.
\newblock \emph{Journal of Algorithms}, 27\penalty0 (1):\penalty0 42--60, 1998.

\bibitem[Frieze and Karo{\'n}ski(2015)]{frieze2015introduction}
Alan Frieze and Micha{\l} Karo{\'n}ski.
\newblock \emph{Introduction to random graphs}.
\newblock Cambridge Press, 2015.

\bibitem[Gammell et~al.(2015)Gammell, Srinivasa, and Barfoot]{gammell2015batch}
Jonathan~D. Gammell, Siddhartha~S. Srinivasa, and Timothy~D. Barfoot.
\newblock {Batch Informed Trees: Sampling-based optimal planning via
  heuristically guided search of random geometric graphs}.
\newblock In \emph{ICRA}, 2015.

\bibitem[Golovin and Krause(2011)]{golovin2011adaptive}
Daniel Golovin and Andreas Krause.
\newblock Adaptive submodularity: Theory and applications in active learning
  and stochastic optimization.
\newblock \emph{Journal of Artificial Intelligence Research}, 2011.

\bibitem[Golovin et~al.(2010)Golovin, Krause, and Ray]{golovin2010near}
Daniel Golovin, Andreas Krause, and Debajyoti Ray.
\newblock Near-optimal bayesian active learning with noisy observations.
\newblock In \emph{NIPS}, 2010.

\bibitem[Hart et~al.(1968)Hart, Nilsson, and Raphael]{hart1968formal}
Peter~E Hart, Nils~J Nilsson, and Bertram Raphael.
\newblock A formal basis for the heuristic determination of minimum cost paths.
\newblock \emph{IEEE Trans. on Systems Science and Cybernetics}, 1968.

\bibitem[Howard(1966)]{howard1966information}
Ronald~A Howard.
\newblock Information value theory.
\newblock \emph{IEEE Tran. Systems Science Cybernetics}, 1966.

\bibitem[Javdani et~al.(2014)Javdani, Chen, Karbasi, Krause, Bagnell, and
  Srinivasa]{javdani2014near}
Shervin Javdani, Yuxin Chen, Amin Karbasi, Andreas Krause, J.~Andrew~(Drew)
  Bagnell, and Siddhartha Srinivasa.
\newblock Near optimal bayesian active learning for decision making.
\newblock In \emph{AISTATS}, 2014.

\bibitem[Kononenko(2001)]{kononenko2001machine}
Igor Kononenko.
\newblock Machine learning for medical diagnosis: History, state of the art and
  perspective.
\newblock \emph{Artificial Intelligence in Medicine}, 2001.

\bibitem[Krause and Guestrin(2009)]{krause2009optimal}
Andreas Krause and Carlos Guestrin.
\newblock Optimal value of information in graphical models.
\newblock \emph{Journal of Artificial Intelligence Research}, 35:\penalty0
  557--591, 2009.

\bibitem[LaValle(2006)]{Lav06}
S.~M. LaValle.
\newblock \emph{Planning Algorithms}.
\newblock Cambridge University Press, Cambridge, U.K., 2006.

\bibitem[Narayanan and Likhachev(2017)]{narayanan2017heuristic}
Venkatraman Narayanan and Maxim Likhachev.
\newblock Heuristic search on graphs with existence priors for
  expensive-to-evaluate edges.
\newblock In \emph{ICAPS}, 2017.

\bibitem[Nielsen and Kavraki(2000)]{nielsen2000two}
Christian~L Nielsen and Lydia~E Kavraki.
\newblock A 2 level fuzzy prm for manipulation planning.
\newblock In \emph{IROS}, 2000.

\bibitem[Penrose(2003)]{penrose2003random}
Mathew Penrose.
\newblock \emph{Random geometric graphs}.
\newblock Oxford University Press, 2003.

\bibitem[Pivtoraiko et~al.(2009)Pivtoraiko, Knepper, and
  Kelly]{pivtoraiko2009differentially}
Mihail Pivtoraiko, Ross~A Knepper, and Alonzo Kelly.
\newblock Differentially constrained mobile robot motion planning in state
  lattices.
\newblock \emph{Journal of Field Robotics}, 2009.

\bibitem[Tallavajhula et~al.(2016)Tallavajhula, Choudhury, Scherer, and
  Kelly]{tallavajhula2016list}
Abhijeet Tallavajhula, Sanjiban Choudhury, Sebastian Scherer, and Alonzo Kelly.
\newblock List prediction applied to motion planning.
\newblock In \emph{ICRA}, 2016.

\bibitem[Yoshizumi et~al.(2000)Yoshizumi, Miura, and
  Ishida]{yoshizumi2000partial}
Takayuki Yoshizumi, Teruhisa Miura, and Toru Ishida.
\newblock A* with partial expansion for large branching factor problems.
\newblock In \emph{AAAI/IAAI}, pages 923--929, 2000.

\end{thebibliography}

\newpage
\appendix

\begin{appendices}

\section{Proof of Lemma \ref{lem:ec2}}
\label{sec:proof:lem_ec2}
\begin{lemma*} 
The expression $\fec{\obsOutcome}$ is strongly adaptive monotone and adaptive submodular.
\end{lemma*}

\begin{proof}
The proof for $\fec{\obsOutcome}$ is a straight forward application of Lemma 5 from \citet{golovin2010near}. We now adapt the proof of adaptive submodularity from Lemma 6 in \citet{golovin2010near}

We first prove the result for uniform prior. To prove adaptive submodularity, we must show that for all $\outcomeTestSet{\mathcal{A}} < \outcomeTestSet{\mathcal{B}}$ and $\test \in \testSet$ , we have $\gain{\ectext}{\test}{\outcomeTestSet{\mathcal{A}}} \geq \gain{\ectext}{\test}{\outcomeTestSet{\mathcal{B}}}$. Fix $\test$ and $\obsOutcome$, and let $\mathcal{V}(\obsOutcome) = \setst{\hyp}{P(\hyp | \obsOutcome) > 0} $ denote the version space, if $\obsOutcome$ encodes the observed outcomes. Let $\nv = \abs{\mathcal{V}(\obsOutcome)}$ be the number of hypotheses in the version space. 
Likewise, let $n_{i,a}(\obsOutcome) = \abs{ \set{\hyp : \hyp \in \mathcal{V}(\obsOutcome, \outcomeVarTest{\test} = a) \cap \hypSpaceR_i} }$, 
and let $n_a(\obsOutcome) = \sum\limits_{i=1}^{l} n_{i,a}(\obsOutcome)$. 
We define a function $\phi$ of the quantities ${n_{i,a} : 1 \leq i \leq l,a \in \{0,1\}}$ such that $\gain{\ectext}{\test}{\obsOutcome} = \phi(n(\obsOutcome))$, where $n(\obsOutcome)$ is the vector consisting of $n_{i,a}(\obsOutcome)$ for all $i$ and $a$. For brevity, we suppress the dependence of $\obsOutcome$ where it is unambiguous.

It will be convenient to define $e_a$ to be the number of edges cut by $\test$ such that at $\test$ both hypotheses agree with each other but disagree with the realized hypothesis $\hyp*$, conditioning on $\outcomeVarTest{\test} = a$. Written as a function of $n$, we have  $e_a = \sum\limits_{i < j}\sum\limits_{b \neq a} n_{i,b} n_{j,b}$.

We also define $\gamma_a$ to be the number of edges cut by $\test$ corresponding to self-edges belonging to hypotheses that disagree with the realized hypothesis $\hyp^*$, conditioning on $\outcomeVarTest{\test} = a$. Written as a function of $n$, we have  $\gamma_a = \sum\limits_{i} \sum\limits_{b \neq a} n_{i,b}^2$. 

\begin{equation}
	\phi(\nvec) = \sum\limits_{i < j} \sum\limits_{a \neq b} n_{i,a} n_{j,b} + \sum\limits_a e_a \left( \frac{n_a}{\nv} \right) + \sum\limits_a \gamma_a \left( \frac{n_a}{\nv} \right)
\end{equation}

where $e_a = \sum\limits_{i < j}\sum\limits_{b \neq a} n_{i,b} n_{j,b}$ and $\gamma_a = \sum\limits_{i} \sum\limits_{b \neq a} n_{i,b}^2$.
Here, $i$ and $j$ range over all class indices, and $a$ and $b$ range over all possible outcomes of test $t$. The first term on the right-hand side counts the number of edges that will be cut by selecting test $t$ no matter what the outcome of $t$ is. Such edges consist of hypotheses that disagree with each other at $t$ and, as with all edges, lie in different classes. The second term counts the expected number of edges cut by $t$ consisting of hypotheses that agree with each other at $t$. Such edges will be cut by $t$ iff they disagree with $h^*$ at $t$. The third term counts the expected number of edges cut by $t$ consisting of hypothesis with self-edges that disagree with $h^*$ at $t$.

We need to show $\frac{\partial \phi}{\partial n_{k,c}} \geq 0$ according to proof of Lemma 6 in \citet{golovin2010near}.

\begin{equation}
	\label{eq:ec2:deriv}
	\frac{\partial \phi}{\partial n_{k,c}} = 
	\frac{\partial}{\partial n_{k,c}} \left( \sum\limits_{i < j} \sum\limits_{a \neq b} n_{i,a} n_{j,b}  \right) +
	\frac{\partial}{\partial n_{k,c}} \left( \sum\limits_a e_a \left( \frac{n_a}{\nv} \right) \right) + 
	\frac{\partial}{\partial n_{k,c}} \left( \sum\limits_a \gamma_a \left( \frac{n_a}{\nv} \right) \right)
\end{equation}

Expanding the first term in (\ref{eq:ec2:deriv})
\begin{equation}
	\frac{\partial}{\partial n_{k,c}} \left( \sum\limits_{i < j} \sum\limits_{a \neq b} n_{i,a} n_{j,b}  \right)
	= \sum\limits_{i \neq k, a \neq c} n_{i,a}
\end{equation}

Expanding the second term in (\ref{eq:ec2:deriv})
\begin{equation}
	\frac{\partial}{\partial n_{k,c}} \left( \sum\limits_a e_a \left( \frac{n_a}{\nv} \right) \right)
	= \sum\limits_{i \neq k, a \neq c} \frac{ n_a n_{i,c} }{\nv} - \sum\limits_{b} \frac{e_b n_b}{\nv^2} + \frac{e_c}{\nv}
\end{equation}

Expanding the third term in (\ref{eq:ec2:deriv})
\begin{equation}
\begin{aligned}
	\frac{\partial}{\partial n_{k,c}} \left( \sum\limits_a \gamma_a \left( \frac{n_a}{\nv} \right) \right)
	&= \frac{\partial}{\partial n_{k,c}} \left( \frac{n_c}{\nv} \gamma_c \right) + \sum\limits_{a \neq c} \frac{\partial}{\partial n_{k,c}} \left( \frac{n_a}{\nv} \gamma_a \right) \\
	&= \frac{n_c}{\nv} \underbrace{\frac{\partial}{\partial n_{k,c}} \gamma_c}_{= 0} + \frac{\gamma_c}{\nv}  \underbrace{\frac{\partial}{\partial n_{k,c}} n_c}_{= 1} - \gamma_c n_c \frac{\partial}{\partial n_{k,c}} \left( \frac{1}{\nv} \right) + \sum\limits_{a \neq c} \frac{\partial}{\partial n_{k,c}} \left( \frac{n_a}{\nv} \gamma_a \right)\\
	&=  \frac{\gamma_c}{\nv} - \frac{ \gamma_c n_c }{\nv^2} +  \sum\limits_{a \neq c} \frac{\partial}{\partial n_{k,c}} \left( \frac{n_a}{\nv} \gamma_a \right)\\
	&=  \frac{\gamma_c}{\nv} - \frac{ \gamma_c n_c }{\nv^2} +  
	\sum\limits_{a \neq c} \left(
	\frac{n_a}{\nv} \underbrace{\left( \frac{\partial}{\partial n_{k,c}} \gamma_a \right)}_{= 2n_{k,c}} + 
	\frac{\gamma_a}{\nv} \underbrace{ \frac{\partial}{\partial n_{k,c}} n_a }_{=0} + 
	\gamma_a n_a \frac{\partial}{\partial n_{k,c}} \frac{1}{\nv}
	\right) \\
	&=  \frac{\gamma_c}{\nv} - \frac{ \gamma_c n_c }{\nv^2} +  
	\sum\limits_{a \neq c} \left( 2\frac{n_a n_{k,c}}{\nv} - \frac{\gamma_a n_a}{\nv^2}  \right) \\
	&= \frac{\gamma_c}{\nv} + 2 n_{k,c} \sum\limits_{a \neq c} \frac{n_a}{\nv} - \sum\limits_b \frac{\gamma_b n_b}{\nv^2} \\
\end{aligned}
\end{equation}

Putting it all together

\begin{equation}
	\label{eq:ec2:deriv2}
	\frac{\partial \phi}{\partial n_{k,c}} = 
	\frac{(e_c + \gamma_c)}{\nv} + 
	\sum\limits_{i \neq k, a \neq c} \frac{ n_a n_{i,c} }{\nv} +
	2 n_{k,c} \sum\limits_{a \neq c} \frac{n_a}{\nv} +
	\sum\limits_{i \neq k, a \neq c} n_{i,a}
	- \sum\limits_{b} \frac{e_b n_b}{\nv^2} - \sum\limits_b \frac{\gamma_b n_b}{\nv^2}
\end{equation}

Multiplying (\ref{eq:ec2:deriv2}) by $\nv$ we see it is non negative iff

\begin{equation}
	\sum\limits_b \frac{(e_b + \gamma_b) n_b}{\nv} \leq e_c + \gamma_c + \sum\limits_{a \neq c, i \neq k} n_a n_{i,c} + 2 n_{k,c} \sum\limits_{a \neq c} n_a + \nv \sum\limits_{a \neq c, i \neq k} n_{i,a} 
\end{equation}

Expanding LHS we get

\begin{equation}
\begin{aligned}
\sum\limits_b \frac{(e_b + \gamma_b) n_b}{\nv} &=  \frac{(e_c + \gamma_c) n_c}{\nv} + \sum\limits_{b \neq c} \frac{(e_b + \gamma_b) n_b}{\nv} \\
&\leq e_c + \frac{\gamma_c n_c}{\nv} + \sum\limits_{b \neq c} \frac{(e_b + \gamma_b) n_b}{\nv} \\
&\leq e_c + \frac{\gamma_c n_c}{\nv} + \sum\limits_{b \neq c} \frac{n_b}{\nv} \left( \sum\limits_{i<j}\sum\limits_{a \neq b} n_{i,a} . n_{j,a} + \sum\limits_i \sum\limits_{a\neq b} n_{i,a}^2 \right) \\
&\leq e_c + \frac{\gamma_c n_c}{\nv} + \sum\limits_{b \neq c} \frac{n_b}{\nv} \left( \sum\limits_{i<j} n_{i,c} . n_{j,c} + \sum\limits_i n_{i,c}^2 \right) + \sum\limits_{b \neq c} \frac{n_b}{\nv} \left( \sum\limits_{i<j}\sum\limits_{a \neq b,c} n_{i,a} . n_{j,a} + \sum\limits_i \sum\limits_{a\neq b,c} n_{i,a}^2 \right)\\
&\leq e_c + 
\underbrace{\sum\limits_{b \neq c} \frac{n_b}{\nv} \left( \sum\limits_{i<j} n_{i,c} . n_{j,c} + \sum\limits_i n_{i,c}^2 \right)}_{\text{\textcircled{A}}} + \underbrace{\sum\limits_{b \neq c} \frac{n_b}{\nv} \left( \sum\limits_{i<j}\sum\limits_{a \neq b,c} n_{i,a} . n_{j,a} + \sum\limits_i \sum\limits_{a\neq b,c} n_{i,a}^2 \right) + \frac{\gamma_c n_c}{\nv}}_{\text{\textcircled{B}}} \\
\end{aligned}
\end{equation}

If $\{x_i\}_{i \geq 0}$ be a finite sequence of non-negative real numbers. Then for any $k$
\begin{equation} 
\label{eq:sumprod}
	\sum\limits_{i < j} x_i x_j + \sum\limits_i x_i^2 \leq \left( \sum_i x_i \right) \left( \sum_{i \neq k} x_i \right) + x_k^2
\end{equation}

Using (\ref{eq:sumprod}) and expanding \textcircled{A} we have
\begin{equation}
\label{eq:terma}
\begin{aligned}
&\sum\limits_{b \neq c} \frac{n_b}{\nv} \left( \sum\limits_{i<j} n_{i,c} . n_{j,c} + \sum\limits_i n_{i,c}^2 \right) \\
\leq&\sum\limits_{b \neq c} \frac{n_b}{\nv} \left( \left( \sum_i n_{i,c} \right) \left( \sum_{i \neq k} n_{i,c} \right) + n_{k,c}^2 \right) \\
\leq&\sum\limits_{b \neq c} \frac{n_b}{\nv} \left( n_c \left( \sum_{i \neq k} n_{i,c} \right) + n_{k,c}^2 \right) \\
\leq& \sum\limits_{b \neq c} \frac{n_b n_c}{\nv} \left( \sum_{i \neq k} n_{i,c} \right) + \sum\limits_{b \neq c} \frac{n_b}{\nv} n_{k,c}^2  \\
\leq& \sum\limits_{b \neq c} n_b \left( \sum_{i \neq k} n_{i,c} \right) + \sum\limits_{b \neq c} \frac{n_b n_c}{\nv} n_{k,c}  \\
\leq& \sum\limits_{a \neq c, i \neq k} n_a n_{i,c} + \sum\limits_{b \neq c} n_b n_{k,c}  \\
\leq& \sum\limits_{a \neq c, i \neq k} n_a n_{i,c} + 2 n_{k,c} \sum\limits_{a \neq c} n_a   \\
\end{aligned}
\end{equation}

Using (\ref{eq:sumprod}) and expanding \textcircled{B} we have
\begin{equation}
\label{eq:termb}
\begin{aligned}
& \sum\limits_{b \neq c} \frac{n_b}{\nv} \left( \sum\limits_{i<j}\sum\limits_{a \neq b,c} n_{i,a} . n_{j,a} + \sum\limits_i \sum\limits_{a\neq b,c} n_{i,a}^2 \right) + \frac{\gamma_c n_c}{\nv} \\
\leq& \sum\limits_{b \neq c} \frac{n_b}{\nv} \left( 
\sum\limits_{a \neq b,c} \left( \sum_{i \neq k} n_{i,a} \right) n_a
+ \sum\limits_{a\neq b,c} n_{k,a}^2 \right) + \frac{\gamma_c n_c}{\nv} \\
\leq&\sum\limits_{b \neq c} \frac{n_b}{\nv} \nv \sum\limits_{a \neq c} \sum_{i \neq k} n_{i,a} +
\sum\limits_{b \neq c} \frac{n_b}{\nv} \sum\limits_{a\neq b,c} n_{k,a}^2  + \frac{\gamma_c n_c}{\nv} \\
\leq&\sum\limits_{b \neq c} n_b \sum\limits_{a \neq c, i \neq k} n_{i,a} +
\sum\limits_{b \neq c} \frac{n_b}{\nv} \sum\limits_{a\neq c} n_{k,a}^2  + \frac{\gamma_c n_c}{\nv} \\
\leq&\nv \sum\limits_{a \neq c, i \neq k} n_{i,a} +
\sum\limits_{b \neq c} \frac{n_b}{\nv} \gamma_c  + \frac{\gamma_c n_c}{\nv} \\
\leq&\nv \sum\limits_{a \neq c, i \neq k} n_{i,a} +
\gamma_c \left( \sum\limits_{b \neq c} \frac{n_b}{\nv}   + \frac{n_c}{\nv} \right)\\
\leq&\nv \sum\limits_{a \neq c, i \neq k} n_{i,a} +
\gamma_c \\
\end{aligned}
\end{equation}

Combining (\ref{eq:terma}) and (\ref{eq:termb})

\begin{equation}
\begin{aligned}
\sum\limits_b \frac{(e_b + \gamma_b) n_b}{\nv} &\leq e_c + \sum\limits_{a \neq c, i \neq k} n_a n_{i,c} + 2 n_{k,c} \sum\limits_{a \neq c} n_a + 
\nv \sum\limits_{a \neq c, i \neq k} n_{i,a} +
\gamma_c
\end{aligned}
\end{equation}

Hence the inequality $\frac{\partial \phi}{\partial n_{k,c}} \geq 0$ holds. For non-uniform prior, the proofs from Lemma 6 in \citet{golovin2010near} carry over. 

\end{proof}

\section{Proof of Lemma \ref{lem:drd}}
\label{sec:proof:lem_drd}
\begin{lemma} 
The expression $\fdrd{\obsOutcome}$ is strongly adaptive monotone and adaptive submodular.
\end{lemma}

\begin{proof}
We adapt the proof from Lemma 1 in \citet{chen2015submodular}. $\fdrd{\obsOutcome}$ can be shown to be strongly adaptive monotone from \citet{chen2015submodular} by showing each individual $\feci{i}{\obsOutcome}$ is strongly adaptive monotone. 

To proof adaptive submodularity, we must show that for all $\outcomeTestSet{\mathcal{A}} < \outcomeTestSet{\mathcal{B}}$ and $\test \in \testSet$ , we have $\gain{\drd}{\test}{\outcomeTestSet{\mathcal{A}}} \geq \gain{\drd}{\test}{\outcomeTestSet{\mathcal{B}}}$. We first show this for two problems in the noisy OR formulation. 

As shown in (7) in \citet{chen2015submodular}, we have
\begin{equation}
\begin{aligned}
\gain{\drd}{\test}{\outcomeTestSet{\mathcal{A}}} = (1 - \feci{1}{\obsOutcome})\expect{x_t}{\delta_2(x_t | \obsOutcome) | \obsOutcome} + \expect{x_t}{(1 - \feci{2}{\obsOutcomeAdd{\test}})\delta_1(x_t | \obsOutcome) | \obsOutcome}
\end{aligned}
\end{equation}

The first term satisfies
\begin{equation}
	(1 - \feci{1}{\obsOutcome})\expect{x_t}{\delta_2(x_t | \obsOutcome) | \obsOutcome} \geq (1 - \feci{1}{\outcomeTestSet{\mathcal{B}}})\expect{x_t}{\delta_2(x_t | \outcomeTestSet{\mathcal{B}}) | \outcomeTestSet{\mathcal{B}}}
\end{equation}

Let the second term be $\lambda(\nvec)$ and denote $h(\nvec) = (1 - \feci{2}{\obsOutcomeAdd{\test}})$. We will show $\frac{\partial \lambda(\nvec)}{\partial n_{k,c}} \geq 0$ for all $n_{k,c}$.

\begin{equation}
\begin{aligned}
\lambda(\nvec) &= \expect{x_t}{h(\nvec)\delta_1(x_t | \obsOutcome) | \obsOutcome} 
\end{aligned}
\end{equation}

The partial derivative $\frac{\partial \lambda(\nvec)}{\partial n_{k,c}}$ can be expressed as 

\begin{equation}
	\frac{\partial \lambda(\nvec)}{\partial n_{k,c}} = \sum\limits_a \frac{\partial h(\nvec)}{\partial n_{k,c}} \left( \frac{n_a}{\nv} p + \frac{n_a}{\nv} e_a + \frac{n_a}{\nv} \gamma_a \right) + \sum\limits_a h(\nvec) \frac{\partial}{\partial n_{k,c}} \left( \frac{n_a}{\nv} p + \frac{n_a}{\nv} e_a + \frac{n_a}{\nv} \gamma_a \right)
\end{equation}

Since $\frac{\partial h(\nvec)}{\partial n_{k,c}} \geq 0$, the first term is $\geq 0$. Expanding the second term, we have

\begin{equation}
\begin{aligned}
\label{eq:h_deriv}
&\sum\limits_a h(\nvec) \frac{\partial}{\partial n_{k,c}} \left( \frac{n_a}{\nv} p + \frac{n_a}{\nv} e_a + \frac{n_a}{\nv} \gamma_a \right) \\
&= h(\nvec) \underbrace{\frac{\partial}{\partial n_{k,c}} \left( \frac{n_c}{\nv} p + \frac{n_c}{\nv} e_c + \frac{n_c}{\nv} \gamma_c \right)}_{\encircle{A}} + 
\sum\limits_{a \neq c} h(\nvec) \underbrace{\frac{\partial}{\partial n_{k,c}} \left( \frac{n_a}{\nv} p + \frac{n_a}{\nv} e_a + \frac{n_a}{\nv} \gamma_a \right)}_{\encircle{B}}
\end{aligned}
\end{equation}

Expanding first term in $\encircle{A}$ 
\begin{equation}
\begin{aligned}
\frac{\partial}{\partial n_{k,c}} \left( \frac{n_a}{\nv} p \right)
&=\frac{n_c}{\nv} \underbrace{\frac{\partial}{\partial n_{k,c}} p}_{ = \sum\limits_{i \neq k, b \neq c} n_{i,b}} + 
\frac{p}{\nv} \underbrace{\frac{\partial}{\partial n_{k,c}} n_c}_{ = 1} +
p n_c  \frac{\partial}{\partial n_{k,c}} \left( \frac{1}{\nv} \right) \\
&= \frac{n_c}{\nv} \sum\limits_{i \neq k, b \neq c} n_{i,b} + \frac{p}{\nv} \left( 1 - \frac{n_c}{\nv}\right)
\end{aligned}
\end{equation}

Expanding second term in $\encircle{A}$ 
\begin{equation}
\begin{aligned}
\frac{\partial}{\partial n_{k,c}} \left( \frac{n_c}{\nv} e_c \right)
&=\frac{n_c}{\nv} \underbrace{\frac{\partial}{\partial n_{k,c}} e_c}_{ = 0} + 
\frac{e_c}{\nv} \underbrace{\frac{\partial}{\partial n_{k,c}} n_c}_{ = 1} +
e_c n_c  \frac{\partial}{\partial n_{k,c}} \left( \frac{1}{\nv} \right) \\
&= \frac{e_c}{\nv} \left( 1 - \frac{n_c}{\nv}\right)
\end{aligned}
\end{equation}

Expanding second term in $\encircle{A}$ 
\begin{equation}
\begin{aligned}
\frac{\partial}{\partial n_{k,c}} \left( \frac{n_c}{\nv} \gamma_c \right)
&=\frac{n_c}{\nv} \underbrace{\frac{\partial}{\partial n_{k,c}} \gamma_c}_{ = 0} + 
\frac{\gamma_c}{\nv} \underbrace{\frac{\partial}{\partial n_{k,c}} n_c}_{ = 1} +
\gamma_c n_c  \frac{\partial}{\partial n_{k,c}} \left( \frac{1}{\nv} \right) \\
&= \frac{\gamma_c}{\nv} \left( 1 - \frac{n_c}{\nv}\right)
\end{aligned}
\end{equation}

Putting things together $\encircle{A}$ evaluates to
\begin{equation}
\begin{aligned}
\frac{n_c}{\nv} \sum\limits_{i \neq k, b \neq c} n_{i,b} + \frac{p}{\nv} \left( 1 - \frac{n_c}{\nv}\right) + 
\frac{e_c}{\nv} \left( 1 - \frac{n_c}{\nv}\right) +
\frac{\gamma_c}{\nv} \left( 1 - \frac{n_c}{\nv}\right)
\end{aligned}
\end{equation}

Now the first term in $\encircle{B}$ evaluates to 
\begin{equation}
\begin{aligned}
\frac{\partial}{\partial n_{k,c}} \left( \frac{n_a}{\nv} p \right) 
&= \frac{n_a}{\nv} \underbrace{ \frac{\partial}{\partial n_{k,c}} p }_{= \sum\limits_{i \neq k, b \neq c} n_{i,b}} +
\frac{p}{\nv} \underbrace{ \frac{\partial}{\partial n_{k,c}} n_a }_{= 1} +
p n_a \frac{\partial}{\partial n_{k,c}} \left( \frac{1}{\nv} \right) \\
&= \frac{n_a}{\nv} \sum\limits_{i \neq k, b \neq c} n_{i,b} - \frac{p n_a}{\nv^2}
\end{aligned}
\end{equation}

The second term in $\encircle{B}$ evaluates to 
\begin{equation}
\begin{aligned}
\frac{\partial}{\partial n_{k,c}} \left( \frac{n_a}{\nv} e_a \right) 
&= \frac{n_a}{\nv} \underbrace{ \frac{\partial}{\partial n_{k,c}} e_a }_{= \sum\limits_{i \neq k} n_{i,c}} +
\frac{e_a}{\nv} \underbrace{ \frac{\partial}{\partial n_{k,c}} n_a }_{= 0} +
e_a n_a \frac{\partial}{\partial n_{k,c}} \left( \frac{1}{\nv} \right) \\
&=\frac{n_a}{\nv} \left( \sum\limits_{i \neq k} n_{i,c} - \frac{e_a}{\nv} \right)
\end{aligned}
\end{equation}

The third term in $\encircle{B}$ evaluates to 
\begin{equation}
\begin{aligned}
\frac{\partial}{\partial n_{k,c}} \left( \frac{n_a}{\nv} \gamma_a \right) 
&= \frac{n_a}{\nv} \underbrace{ \frac{\partial}{\partial n_{k,c}} \gamma_a }_{= 2n_{k,c}} +
\frac{\gamma_a}{\nv} \underbrace{ \frac{\partial}{\partial n_{k,c}} n_a }_{= 0} +
\gamma_a n_a \frac{\partial}{\partial n_{k,c}} \left( \frac{1}{\nv} \right) \\
&=\frac{n_a}{\nv} \left( 2n_{k,c} - \frac{\gamma_a}{\nv} \right)
\end{aligned}
\end{equation}

Combining $\encircle{B}$ 
\begin{equation}
\begin{aligned}
\frac{n_a}{\nv} \sum\limits_{i \neq k, b \neq c} n_{i,b} - \frac{p n_a}{\nv^2} +
\frac{n_a}{\nv} \left( \sum\limits_{i \neq k} n_{i,c} - \frac{e_a}{\nv} \right) +
\frac{n_a}{\nv} \left( \sum\limits_{i \neq k} 2n_{k,c} - \frac{\gamma_a}{\nv} \right) \\
\frac{n_a}{\nv} \left( \sum\limits_{i \neq k, b \neq c} n_{i,b} + \sum\limits_{i \neq k} n_{i,c} + 2n_{k,c} - 
\frac{1}{\nv} \left( p + e_a + \gamma_a\right) \right)
\end{aligned}
\end{equation}

Combining $\encircle{A}$ and $\encircle{B}$, (\ref{eq:h_deriv}) can be evaluated as
\begin{equation}
\begin{aligned}
\label{eq:h_deriv2}
& h(\nvec) \left( \underbrace{\frac{n_c}{\nv} \sum\limits_{i \neq k, b \neq c} n_{i,b}}_{ \geq 0} 
+ \underbrace{\frac{p}{\nv} \left( 1 - \frac{n_c}{\nv}\right)}_{\geq 0} + 
\underbrace{\frac{e_c}{\nv} \left( 1 - \frac{n_c}{\nv}\right)}_{\geq 0} +
\frac{\gamma_c}{\nv} \left( 1 - \frac{n_c}{\nv}\right)
\right) + \\
&\sum\limits_{a \neq c} h(\nvec) \frac{n_a}{\nv} \left( \sum\limits_{i \neq k, b \neq c} n_{i,b} + \sum\limits_{i \neq k} n_{i,c} +  2n_{k,c} - 
\frac{1}{\nv} \left( p + e_a + \gamma_a\right) \right) \\
& \geq h(\nvec) \frac{\gamma_c}{\nv} \left( 1 - \frac{n_c}{\nv}\right) +
\sum\limits_{a \neq c} h(\nvec) \frac{n_a}{\nv} \left( \sum\limits_{i \neq k, b } n_{i,b} + 2n_{k,c} - 
\frac{1}{\nv} \left( p + e_a + \gamma_a\right) \right) \\
& \geq h(\nvec) \gamma_c \left( 1 - \frac{n_c}{\nv}\right) +
\sum\limits_{a \neq c} h(\nvec) \frac{n_a}{\nv} \left( \nv \left( \sum\limits_{i \neq k, b } n_{i,b} + 2n_{k,c} \right) - 
\underbrace{\left( p + e_a + \gamma_a\right)}_{\encircle{C}} \right) \\
\end{aligned}
\end{equation}

Expanding $\encircle{C}$ we have
\begin{equation}
\begin{aligned}
( p + e_a + \gamma_a ) &\leq \sum\limits_{i<j} \sum\limits_{b\neq d} n_{i,b} n_{j,d} + \sum\limits_{i<j} \sum\limits_{b\neq a,c} n_{i,b} n_{j,b} + \sum\limits_{i} \sum\limits_{b\neq a,c} n_{i,b}^2 \\
&\leq \sum\limits_{i<j} \sum\limits_{b\neq d} n_{i,b} n_{j,d} + \sum\limits_{i<j} \sum\limits_{b\neq d} n_{i,b} n_{j,b} + \sum\limits_{i \neq k} \sum\limits_{b\neq a,c} n_{i,b}^2 + \sum\limits_{b\neq a,c} n_{k,b}^2 \\
&\leq \left( \sum\limits_{i,d} n_{i,d} \right) \left( \sum\limits_{j \neq k} \sum\limits_b n_{j,b} \right) +  \sum\limits_{b\neq a,c} n_{k,b}^2 \\
&\leq \nv \left( \sum\limits_{j \neq k} \sum\limits_b n_{j,b} \right) + \gamma_c 
\end{aligned}
\end{equation}

Substituting in (\ref{eq:h_deriv2}) we have
\begin{equation}
\begin{aligned}
& \geq h(\nvec) \gamma_c \left( 1 - \frac{n_c}{\nv}\right) +
\sum\limits_{a \neq c} h(\nvec) \frac{n_a}{\nv} \left( \nv \left( \sum\limits_{i \neq k, b } n_{i,b} + 2n_{k,c} \right) - 
\left( p + e_a + \gamma_a\right) \right) \\
& \geq h(\nvec) \gamma_c \left( 1 - \frac{n_c}{\nv}\right) +
\sum\limits_{a \neq c} h(\nvec) \frac{n_a}{\nv} \left( \nv \left( \sum\limits_{i \neq k, b } n_{i,b} + 2n_{k,c} \right) - 
\nv \left( \sum\limits_{j \neq k} \sum\limits_b n_{j,b} \right) + \gamma_c  \right) \\
& \geq h(\nvec) \gamma_c \left( 1 - \frac{n_c}{\nv}\right) +
\sum\limits_{a \neq c} h(\nvec) \frac{n_a}{\nv} \gamma_c \\
& \geq 0 \\
\end{aligned}
\end{equation}

Hence we have proved $\frac{\partial \lambda(\nvec)}{\partial n_{k,c}} \geq 0$ for all $n_{k,c}$. This implies adaptive submodularity is proved for $2$ regions. For more than $2$, we apply the recursive technique in Lemma 1 in  \citet{chen2015submodular}.

\end{proof}

\section{Proof of Theorem \ref{eq:drd_near_opt}}
\label{proof:bisect_nearopt}
\begin{theorem*}
Let $\numRegion$ be the number of regions, $\pminH$ the minimum prior probability of any hypothesis, $\policy_{DRD}$ be the greedy policy and $\policyOpt$ with the optimal policy. Then $\cost(\policy_{DRD}) \leq \cost(\policy^*)(2\numRegion \log \frac{1}{\pminH} + 1)$.
\end{theorem*}

\begin{proof}
This is a straightforward application of Theorem 2 in \citet{chen2015submodular}.
\end{proof}

\section{Proof of Theorem \ref{thm:max_prob}}
\label{sec:proof:max_prob}
\begin{theorem*}
A policy that greedily latches to a region according the the posterior conditioned on the region outcomes has a near-optimality guarantee of 4 w.r.t the optimal region evaluation sequence.
\end{theorem*}

\begin{proof}
We establish an equivalence to the problem of greedy search on a binary vector as described in \citet{dor1998greedy}.

\begin{problem}
Consider the $n$-dimensional binary space with some (general) probability distribution. Suppose that a random vector is sampled from this space and it is initially unseen. A search algorithm on such a vector is a procedure inspecting one coordinate at a time in a pre-determined order. It terminates when a 1-coordinate is found or when all coordinates were tested and found to be 0. A greedy search is one that goes at each stage to the next coordinate most likely to be 1, taking into account the findings of the previous examinations and the distribution. Can we bound the performance of the greedy search with a search optimal in expectation? 
\end{problem}

\citet{dor1998greedy} proves the following
\begin{theorem}
The expectation of the greedy algorithm (denoted $E_1$) is always less than $4$ times the expectation of the optimal algorithm (denoted $E_2$)
\end{theorem}

If we imagine each region $\region_i$ to be a coordinate, then an algorithm that greedily selects regions to evaluate based on the outcomes of previous region check has a bounded sub-optimality. We note that \algMaxProbReg uses a more accurate posterior as it has access to the individual results of edge evaluation. Hence it is expected to do better than the greedy algorithm that only conditions on the outcome of the region evaluation.

\end{proof}

\section{Proof of Theorem \ref{thm:sub_opt}}
\label{sec:proof:sub_opt}

\begin{theorem*}
  Let $\pmin = \min_i P(\region_i)$, $\pminH = \min_{\hyp \in \hypSpace} P(\hyp)$ and $l = \max_i \abs{\region_i}$. The policy using (\ref{eq:cand_test_set:maxp}) has a suboptimality of $\alpha \left(2 \numRegion \log \left( \frac{1}{\pminH} \right) + 1 \right)$ where 
$\alpha \leq \left( 1 -   \max \left( (1 - \pmin)^2, \pmin^{\frac{2}{l}} \right) \right)^{-1}$.
\end{theorem*}

\begin{proof}
We start of by defining policies that do not greedily maximize $\fdrd{\obsOutcome}$

\begin{definition}
	Let an $\alpha$-approximate greedy policy be one that selects a test $\test'$ that satisfies the following criteria
	\begin{equation*}
	\gain{\drd}{\test'}{\obsOutcome} \geq \frac{1}{\alpha} \max\limits_{\test} \gain{\drd}{\test}{\obsOutcome}
	\end{equation*}
\end{definition}

We examine the scenarios where cost is uniform $\cost(\test) = 1$ for the sake of simplicity - the proof can be easily extended to non-uniform setting. We refer to the policy using (\ref{eq:cand_test_set:maxp}) as a test constraint as a \algMaxProbReg policy.

The marginal gain is evaluated as follows
\begin{equation}
\begin{aligned}
	\gain{\drd}{\test}{\obsOutcome} &= \mathbb{E}_{\outcomeTest{\test}} \left[ 
       \prod\limits_{r=1}^\numRegion  
      \left(1 - \prod\limits_{i \in (\region_r \cap \selTestSet)} \Ind(\outcomeVarTest{i} = 1) \prod\limits_{j \in (\region_r \setminus \selTestSet)} \biasTest{j} \right) \right.\\
      & - \left. \left( \prod\limits_{r=1}^\numRegion 
      \left(1 - \prod\limits_{i \in (\region_r \cap \selTestSet \cup \test)} \Ind(\outcomeVarTest{i} = 1) \prod\limits_{j \in (\region_r \setminus \selTestSet \cup \test)} \biasTest{j} \right) \right)
      ( \biasTest{t}^{\outcomeTest{\test}} (1-\biasTest{t})^{1-\outcomeTest{\test}} )^{2\sum\limits_{k=1}^{m} \Ind(\test \in \region_k)} \right]
\end{aligned}
\end{equation}

We will now bound $\alpha$, the ratio of marginal gain of the unconstrained greedy policy and \algMaxProbReg. 

\begin{equation}
\begin{aligned}
\alpha &\leq \frac{ \max\limits_{\test \in \candTestSet} \gain{\drd}{\test}{\obsOutcome} } { \min\limits_{\test \in \maxProbTestSet} \gain{\drd}{\test}{\obsOutcome} }
\end{aligned}
\end{equation}

The numerator and denominator contain $\prod\limits_{r=1}^\numRegion \left(1 - \prod\limits_{i \in (\region_r \cap \selTestSet)} \Ind(\outcomeVarTest{i} = 1) \prod\limits_{j \in (\region_r \setminus \selTestSet)} \biasTest{j} \right)$, the posterior probabilities of regions not being valid. Hence we normalize by dividing this term and expressing $\alpha$ in terms of a residual function $\rho (t | \obsOutcome)$

\begin{equation}
\begin{aligned}
\label{eq:subopt_bound}
\alpha &\leq \frac{ 1 - \min\limits_{\test \in \candTestSet} \rho(t | \obsOutcome)  } {1 - \max\limits_{\test \in \maxProbTestSet} \rho(t | \obsOutcome) }
\end{aligned}
\end{equation}

where the residual function $\rho (t | \obsOutcome)$ is
\begin{equation}
	\rho (t | \obsOutcome) = \frac{ \expect{\outcomeTest{\test}}{\left( \prod\limits_{r=1}^\numRegion 
      \left(1 - \prod\limits_{i \in (\region_r \cap \selTestSet \cup \test)} \Ind(\outcomeVarTest{i} = 1) \prod\limits_{j \in (\region_r \setminus \selTestSet \cup \test)} \biasTest{j} \right) \right)
      ( \biasTest{t}^{\outcomeTest{\test}} (1-\biasTest{t})^{1-\outcomeTest{\test}} )^{2\sum\limits_{k=1}^{m} \Ind(\test \in \region_k)}} }
      { \prod\limits_{r=1}^\numRegion \left(1 - \prod\limits_{i \in (\region_r \cap \selTestSet)} \Ind(\outcomeVarTest{i} = 1) \prod\limits_{j \in (\region_r \setminus \selTestSet)} \biasTest{j} \right) }
\end{equation}

\begin{figure}[t]
    \centering
    \includegraphics[page=1,width=\textwidth]{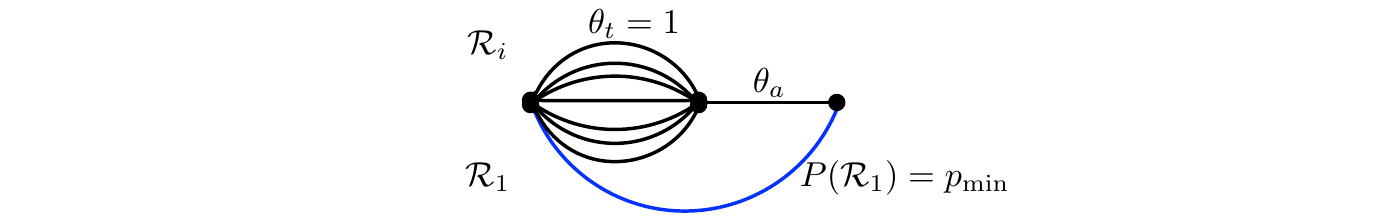}
    \caption{%
    \label{fig:subopt_bound_scenario}
    The scenario where sub-optimality bound is maximized}
\end{figure}%

We will now claim that the bound is maximized in the scenario shown in Fig.~\ref{fig:subopt_bound_scenario}. The most likely region, $R_1$, is an isolated path which contains no tests in common with other regions. Let the probability of this region be $\pmin$ - this is the smallest probability that can be assigned to it. All other $\numRegion-1$ regions (of lower probabilty) share a common test $a$ of probability $\bias_a$. The remaining tests in these regions have probability $1$. Note that $\bias_a \leq \pmin$. In this scenario, the greedy policy will select the common test while the \algMaxProbReg will select a test from the most probably region. 

We will now show that this scenario allows us to realize the upper bound $1$ for the numerator in (\ref{eq:subopt_bound}). In other words, we will show that the residual function $\rho (a | \obsOutcome) \ll 1$ in our scenario.
\begin{equation}
\begin{aligned}
\rho (a | \obsOutcome) & \leq \frac{ \bias_a \pmin \prod\limits_{r=1}^{\numRegion-1} (1 - 1) \bias_a^{2(\numRegion-1)} \;+\; (1 - \bias_a) \pmin \prod\limits_{r=1}^{\numRegion-1} (1 - 0) (1 - \bias_a)^{2(\numRegion-1)} }
{ \pmin \prod\limits_{r=1}^{\numRegion-1} (1 - \bias_a) } \\
& \leq \frac{ (1 - \bias_a)^{2\numRegion - 1} } {  \prod\limits_{r=1}^{\numRegion-1} (1 - \bias_a) } \\
& \leq (1 - \bias_a)^\numRegion 
\end{aligned}
\end{equation}
We can drive $\rho (a | \obsOutcome) \ll 1$ by setting $\numRegion$ arbitrarily high.

We now show that scenario also allows us to bound the denominator in (\ref{eq:subopt_bound}) by maximizing  $\max\limits_{\test \in \region_1} \rho(t | \obsOutcome) $. We first note that by selecting a test that belongs only to one region, the residual is maximized. We have to figure out how large the residual can be. 
Let $\tau = \argmax\limits_{\test \in \region_1} \rho(t | \obsOutcome) $  be the most probable test with probability $\bias_\tau$. Let $\beta = \prod\limits_{\test \in \region_1, \test \neq \tau} \bias_t$ be the lumped probability of all other tests. Note that $\theta_\tau \beta = \pmin$. The residual can be expressed as
\begin{equation}
\begin{aligned}
\label{eq:demon_bound}
\rho (\tau | \obsOutcome) & \leq \frac{ \bias_\tau \prod\limits_{r=1}^{\numRegion-1} (1 - \bias_a) (1 - \beta) \bias_\tau^2 \;+\; (1 - \bias_\tau) \prod\limits_{r=1}^{\numRegion-1} (1 - \bias_a) (1 - \bias_\tau)^2 }
{ \prod\limits_{r=1}^{\numRegion-1} (1 - \bias_a) (1 - \bias_\tau \beta)} \\
& \leq \frac{ \bias_\tau^3 (1-\beta) + (1 - \bias_\tau)^3 } { (1 - \bias_\tau\beta) } \\
& \leq \frac{ \bias_\tau^3 (1-\beta) + \bias_\tau^2 - \bias_\tau^2 + (1 - \bias_\tau)^3 } { (1 - \bias_\tau\beta) } \\
& \leq \frac{ \bias_\tau^2 (1-\bias_\tau \beta) - \bias_\tau^2 (1 - \bias_\tau) + (1 - \bias_\tau)^3 } { (1 - \bias_\tau\beta) } \\
& \leq \bias_\tau^2 - \frac{ (1 - \bias_\tau)(2\bias_\tau - 1) } { (1 - \pmin) } \\
\end{aligned}
\end{equation}

This bound is concave and achieves maxima on the two extrema. In the first case, we assume $\beta = 0$, $\theta_\tau = \pmin$. This leads to 
\begin{equation}
\begin{aligned}
\rho (\tau | \obsOutcome) & \leq \pmin^2 - \frac{ (1 - \pmin)(2\pmin - 1) } { (1 - \pmin) } \\
& \leq \pmin^2 - (2\pmin - 1) \\
& (1 - \pmin)^2
\end{aligned}
\end{equation}

In the second case, we assume $\beta = \theta_\tau$. Lete $l$ be the maximum test in any region. Then $\theta_\tau = \pmin^{\frac{1}{l}}$. This leads to 
\begin{equation}
\begin{aligned}
\rho (\tau | \obsOutcome) & \leq \bias_\tau^2 - \frac{ (1 - \bias_\tau)(2\bias_\tau - 1) } { (1 - \pmin) } \\
& \leq \bias_\tau^2 \\
& \leq  \pmin^{\frac{2}{l}} \\
\end{aligned}
\end{equation}

Combining these we have
\begin{equation}
	\label{eq:rho_bound}
	\rho (\tau | \obsOutcome) \leq \max \left( (1 - \pmin)^2, \pmin^{\frac{2}{l}} \right)
\end{equation}

Substituting (\ref{eq:rho_bound}) in (\ref{eq:subopt_bound}) we have
\begin{equation}
	\label{eq:final_alpha}
	\alpha \leq \frac{1}{ 1 -   \max \left( (1 - \pmin)^2, \pmin^{\frac{2}{l}} \right)} 
\end{equation}

We now use Theorem 11 in \citet{golovin2011adaptive} to state that an $\alpha$-approximate greedy policy $\policy$ optimizing $\fdrd{\obsOutcome}$ enjoys the following guarantee
\begin{equation}
	\cost(\policy) \leq \alpha \cost(\policyOpt) \left(2 \numRegion \log\left(\frac{1}{\pminH}\right) + 1\right) 
\end{equation}

\end{proof}

\section{Proof of Theorem \ref{thm:set_cover}}
\label{sec:proof:set_cover}

\begin{theorem*}
\algSetCover is a near-optimal policy for checking all regions. 
\end{theorem*}

We present a refined version of the theorem that we will prove.

\begin{theorem*}
Let $\policy$ be the \algSetCover policy, which is a partial mapping from observation vector $\obsOutcome$ to tests, such that it terminates only when all regions $\region_i$ are either completely evaluated or invalidated. Let the expected cost of such a policy be $\cost(\policy)$.
Let $\policyOpt$ be the optimal policy for checking all regions. Let $\numTest = \abs{\testSet}$ be the number of tests.
\algSetCover enjoys the following guarantee
\begin{equation*}
	\cost(\policy) \leq \cost(\policyOpt) ( \log (\numTest) + 1) 
\end{equation*}
\end{theorem*}

\begin{proof}
We will prove this by drawing an equivalence of the problem to a special case of \emph{stochastic set coverage} with non-uniform costs, showing \algSetCover greedily solves this problem, and using a guarantee for a greedy policy as presented in \citet{golovin2011adaptive}. 

The stochastic set coverage problem is as follows - there is a ground set of elements $U$, and items $E$ such that item $e$ is associated with a distribution over subsets of $U$. When an item is selected, a set is sampled from its distribution. The problem is to adaptively select items until all elements of $U$ are covered by sampled sets, while minimizing the expected number of items selected. Here we consider the case where a cost is associated with each item.

We now show that the problem of selecting tests to invalidate other tests is equivalent to stochastic set coverage. The ground set is the set of all tests $\testSet$. The item set has a one to one correspondence with the test set $\testSet$. 
Let $\hat{f}(\obsOutcome)$ be the utility function measuring coverage of tests given selected tests and outcomes $\obsOutcome$. This is defined as 
\begin{equation}
	\hat{f}(\obsOutcome) = \abs{ \selTestSet \cup \set{\bigcup\limits_{i=1}^\numRegion \setst{\region_i}{P(\region_i | \obsOutcome) = 0}}} 
\end{equation}
The expected gain in utility when selecting a test $\test$ is as follows - with probability $\bias_\test$ if a $\test$ outcome is true, only $\test$ is covered. With probability $1 - \bias_\test$, if a test outcome is false, tests that belong to regions being invalidated are covered. This can be expressed formally as follows. Given $\test \notin \set{\bigcup\limits_{i=1}^\numRegion \setst{\region_i}{P(\region_i | \obsOutcome) = 0}}$, the expected gain $\gain{\hat{f}}{\test}{\obsOutcome}$ is
\begin{equation}
\begin{aligned}
&\gain{\hat{f}}{\test}{\obsOutcome} = \expect{\outcome_\test}{ \hat{f}(\obsOutcomeAdd{\test}) - \hat{f}(\obsOutcome) } \\
&= \expect{\outcome_\test}{ \abs{ \selTestSet \cup \set{\test} \cup \set{\bigcup\limits_{i=1}^\numRegion \setst{\region_i}{P(\region_i | \obsOutcomeAdd{\test}) = 0}}} -
\abs{ \selTestSet \cup \set{\bigcup\limits_{i=1}^\numRegion \setst{\region_i}{P(\region_i | \obsOutcome) = 0}}} } \\
&= P(\outcomeVarTest{\test} = 1) \times 1 + \\
& P(\outcomeVarTest{\test} = 0) \times \left(1 + 
\abs{  \set{\bigcup\limits_{i=1}^\numRegion \setst{\region_i}{P(\region_i | \obsOutcome) > 0} - 
  \bigcup\limits_{j=1}^\numRegion \setst{\region_j}{P(\region_j | \obsOutcome, \outcomeVarTest{\test} = 0) > 0}   }
  \setminus \set{\selTestSet \cup \set{\test}} } \right)\\
&= 1 + (1 - \bias_{\test})  
\abs{  \set{\bigcup\limits_{i=1}^\numRegion \setst{\region_i}{P(\region_i | \obsOutcome) > 0} - 
  \bigcup\limits_{j=1}^\numRegion \setst{\region_j}{P(\region_j | \obsOutcome, \outcomeVarTest{\test} = 0) > 0}   }
  \setminus \set{\selTestSet \cup \set{\test}} } \\
\end{aligned}
\end{equation}
\algSetCover is an adaptive greedy policy with respect to $\gain{\hat{f}}{\test}{\obsOutcome}$ as shown
\begin{equation}
\begin{aligned}
\optTest &\in \argmaxprob{\test \in \testSet}\; \gain{\hat{f}}{\test}{\obsOutcome} \\
&\in \argmaxprob{\test \in \testSet}\; (1 - \bias_{\test})  
\abs{  \set{\bigcup\limits_{i=1}^\numRegion \setst{\region_i}{P(\region_i | \obsOutcome) > 0} - 
  \bigcup\limits_{j=1}^\numRegion \setst{\region_j}{P(\region_j | \obsOutcome, \outcomeVarTest{\test}=0) > 0}   }
  \setminus \set{\selTestSet \cup \set{\test}} } \\
\end{aligned}
\end{equation}
Note that $\hat{f}(\groundtruth) = \numTest$ is the maximum value the utility can attain. Let $\policyOpt$ be the optimal policy.
Since $\hat{f}(\obsOutcome)$ is a strong adaptive monotone submodular function, we use Theorem 15 in \citet{golovin2011adaptive} to state the following guarantee
\begin{equation*}
	\cost(\policy) \leq \cost(\policyOpt) ( \log (\numTest) + 1) 
\end{equation*}
\end{proof}

\section{Datasets with large disparity in region sizes}
\label{app:max_prob_reg}

In this section, we investigate scenarios where there is a large disparity in region sizes. We will show that in such scenarios, \algMaxProbReg has an arbitrarily poor performance. We will also show that unconstrained \algName vastly outperforms all other algorithms on such problems. 

We first examine the scenario as shown in Fig.~\ref{fig:max_prob_reg_counter_example}(a). There are two regions $\region_1$ and $\region_2$. $\region_1$ has only $1$ test $a$ with bias $\theta_a$. $\region_2$ has $T$ tests $\{ b_1, \dots, b_T \}$, each with bias $\theta_b$. The evaluation cost of each test is $1$. The following condition is enforced
\begin{equation}
	\theta_b^T = \theta_a + \varepsilon	
\end{equation} 

Under such conditions, the \algMaxProbReg algorithm would check tests in $\region_2$ before proceeding to $\region_1$. We compare the performance of this policy to the converse - one that evaluates $\region_1$ and then proceeds to $\region_2$.

Lets analyze the expected cost of \algMaxProbReg. If $\region_2$ is valid, it incurs a cost of $T$, else it incurs a cost of $T+1$. This equates to
\begin{equation}
\begin{aligned}
& \theta_b^T(T) + (1 - \theta_b^T)(T+1) \\
&= (T + 1) - \theta_b^T
\end{aligned}
\end{equation}

We now analyze the converse which selects test $a$. If $\region_1$ is valid, it incurs a cost of $1$, else it incurs $T+1$. This equates to 
\begin{equation}
\begin{aligned}
& \theta_a(1) + (1 - \theta_a)(T+1) \\
&= (T + 1) - \theta_aT
\end{aligned}
\end{equation}

\algMaxProbReg incurs a larger expected cost that equates to
\begin{equation}
\begin{aligned}
& (T + 1) - \theta_b^T - ( (T + 1) - \theta_aT ) \\
&= \theta_aT - \theta_b^T \\
&= \theta_aT - \theta_a - \varepsilon \\
& = \theta_a(T-1) - \varepsilon \\
\end{aligned}
\end{equation}
$T$ can be made arbitrarily large to push this quantity higher. 

We will now show that unconstrained \algName will evaluate $\region_1$ in this case. We apply the \algName selection rule in (\ref{eq:greedy_fdrd}) to this problem. The utility of selecting test $a$ is
\begin{equation}
\begin{aligned}
\gain{\drd}{a}{\obsOutcome} &=
(1 - \theta_a)(1 - \theta_b^T)  - \left[ \theta_a (1 - 1)(1 - \theta_b^T) \theta_a^2 + (1-\theta_a) (1 - 0) (1 - \theta_b^T) (1-\theta_a)^2  \right]
\\
&= (1 - \theta_a)(1 - \theta_b^T) - (1-\theta_a)^3(1 - \theta_b^T) \\
\end{aligned}
\end{equation}

The utility of selecting test $b_i$ is
\begin{equation}
\begin{aligned}
\gain{\drd}{b_i}{\obsOutcome} &=
(1 - \theta_a)(1 - \theta_b^T)  - \left[ \theta_b (1 - \theta_a)(1 - \theta_b^{T-1}) \theta_b^2 + (1-\theta_b) (1 - \theta_a) (1 - 0) (1-\theta_b)^2  \right]
\\
&= (1 - \theta_a)(1 - \theta_b^T) - \theta_b^3 (1 - \theta_a) (1 - \theta_b^{T-1}) -  (1-\theta_b)^3 (1 - \theta_a) \\
&= (1 - \theta_a)(1 - \theta_b^T) - (1 - \theta_a) [\theta_b^3  (1 - \theta_b^{T-1}) -  (1-\theta_b)^3 ] \\
\end{aligned}
\end{equation}

We assume that $T$ is sufficiently large such that $\theta_b \approx 1$. Then the difference is
\begin{equation}
\begin{aligned}
\gain{\drd}{a}{\obsOutcome} -  \gain{\drd}{b_i}{\obsOutcome}&=
(1 - \theta_a) [\theta_b^3  (1 - \theta_b^{T-1}) -  (1-\theta_b)^3 ] - (1-\theta_a)^3(1 - \theta_b^T) \\
&\approx (1 - \theta_a) [(1 - \theta_b^{T-1}) ] - (1-\theta_a)^3(1 - \theta_b^T) \\
&\approx (1 - \theta_a) (1 - \theta_b^T) [1 - (1-\theta_a)^2] \\
&\geq 0\\
\end{aligned}
\end{equation}

\begin{figure}[t]
    \centering
    \includegraphics[page=1,width=\textwidth]{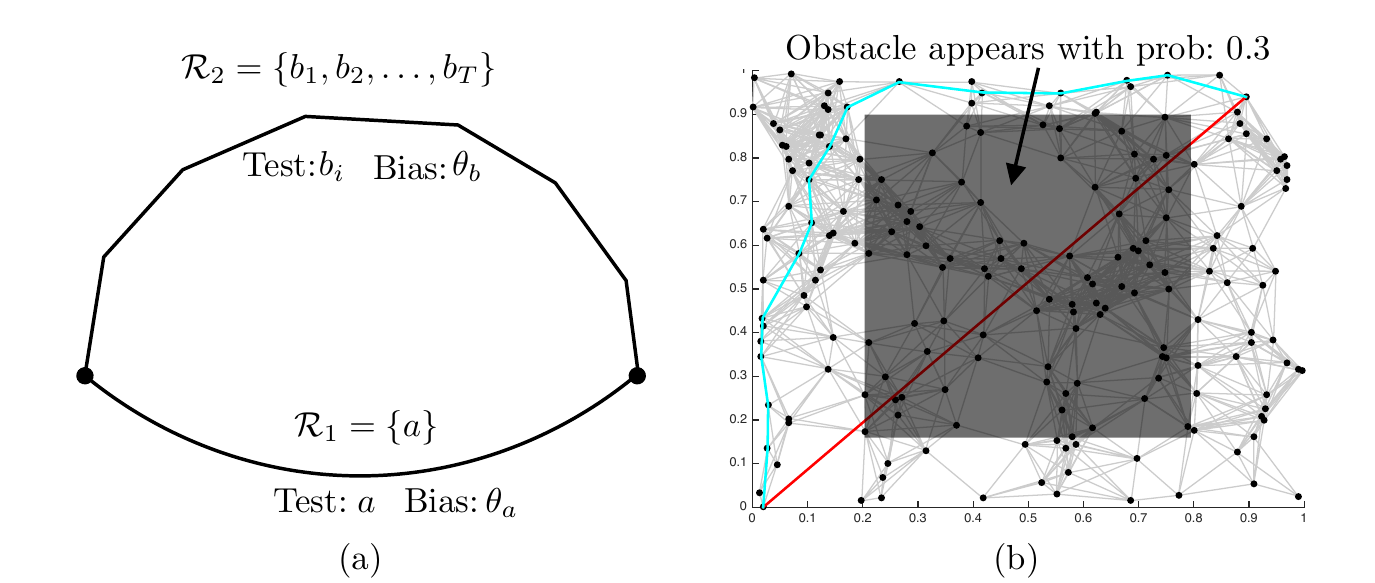}
    \caption{%
    \label{fig:max_prob_reg_counter_example}
    (a) The scenario where unconstrained \algName outperforms \algMaxProbReg significantly. (b) The 2D motion planning scenario where unconstrained \algName outperforms others. Here the graph contains a straight line joining start and goal. With a low-probability, a block is placed between start and goal. This forces the path with many edges circumnavigating the block to have maximum probability. The straight line path joining start and goal has lower probability. Hence there is a similarity to the synthetic example in (a)} 
\end{figure}%

Hence unconstrained \algName would significantly outperform \algMaxProbReg in these problems. We now empirically show this result on a synthetic dataset as well as a carefully constructed 2D motion planning dataset. Table~\ref{tab:max_prob_res} shows a summary of these results. 

\begin{table}[!htpb]
\small
\centering
\caption{Normalized cost ($95\%$ C.I. lower / upper bound) with respect to unconstrained \algName}
\begin{tabulary}{\textwidth}{LCCCCC}\toprule
       & {\bf \algMVOI}       & {\bf \algRandom}       & {\bf \algMaxTally}       & {\bf \algSetCover}    & {\bf \algName} \\ 
       &                      & {Unconstrained}    & {Unconstrained}   & {Unconstrained}   & {Unconstrained}               \\
       &                      & {MaxProbReg}        & {MaxProbReg}      & {MaxProbReg}      & {MaxProbReg}                 \\ \midrule
Synthetic   &                    & $(6.50, 8.00)$  & $(5.50, 6.50)$  & $(3.00, 3.50)$  & $\red(0.00, 0.00)$    \\
($T:10$)    & $(3.00, 3.50)$     & $(3.00, 4.50)$  & $(5.00, 7.50)$  & $(3.00, 3.50)$  & $(3.00, 3.50)$  \\ 
2D Plan     &                    & $(9.50, 11.30)$ & $(2.80, 6.10)$  & $(6.60, 10.50)$  & $\red(0.00, 0.00)$   \\
($\numRegion:2$)  
            & $(6.60, 10.50)$    & $(6.90, 10.80)$ & $(6.80, 8.30)$  & $(6.60, 10.50)$  & $(7.30, 11.20)$  \\ 
2D Plan
            &                    & $(2.44, 3.17)$  & $(2.83, 3.28)$  & $(2.50, 2.56)$  & $\red(0.00, 0.00)$   \\
($\numRegion:19$)
            & $(0.89, 1.17)$	 & $(1.06, 1.28)$  & $(0.89, 0.94)$  & $(0.78, 0.94)$  & $(0.78, 0.89)$   \\ \bottomrule
\end{tabulary}
\label{tab:max_prob_res}
\end{table}

The first dataset is Synthetic which is a instantiation of the scenario shown in Fig.~\ref{fig:max_prob_reg_counter_example}(a). We set $\theta_a = 0.9$, $\theta_b = 0.9906$, $\varepsilon = 0.01$, $T = 10$. We see \algMaxProbReg does incurs $3$ times more cost than the unconstrained variant.

The second dataset is a motion planning dataset as shown in Fig.~\ref{fig:max_prob_reg_counter_example}(b). The dataset is created to closely resemble the synthetic dataset attributes. A RGG graph is created, and the straight line joining start and goal is added to the set of edges. A distribution of obstacles is created by placing a block with probability $0.3$. The first dataset has $2$ regions - the straight line containing one edge, and a path that goes around the block containing many more edges. Unconstrained \algName evaluates the straight line first. \algMaxProbReg evaluates the longer path first. We see \algMaxProbReg does incurs $7$ times more cost than the unconstrained variant.

The third dataset is same as the second, except the number of regions is increased to $19$. Now we see that the contrast reduces. \algMaxProbReg incurs $0.78$ fraction more cost than unconstrained version.

\end{appendices}

\end{document}